\documentclass[english]{article}
\usepackage[numbers]{natbib}

\usepackage[T1]{fontenc}
\usepackage[utf8]{inputenc}
\usepackage{geometry}
\usepackage{float}
\usepackage{amsmath}
\usepackage{amsthm}
\usepackage{amssymb}

\theoremstyle{plain}
\newtheorem{thm}{\protect\theoremname}
\theoremstyle{remark}
\newtheorem{rem}[thm]{\protect\remarkname}
\theoremstyle{plain}
\newtheorem{prop}[thm]{\protect\propositionname}
\theoremstyle{plain}
\newtheorem{lem}[thm]{\protect\lemmaname}
\newtheorem{cor}[thm]{\protect\corollaryname}
\theoremstyle{definition}
\newtheorem{defn}[thm]{\protect\definitionname}

\usepackage{graphicx,color}
\usepackage[noend]{algorithmic}
\usepackage{algorithm}
\usepackage{cellspace}
\usepackage{babel}

\providecommand{\definitionname}{Definition}
\providecommand{\lemmaname}{Lemma}
\providecommand{\propositionname}{Proposition}
\providecommand{\remarkname}{Remark}
\providecommand{\theoremname}{Theorem}
\providecommand{\corollaryname}{Corollary}

 \newif\ifnotes
\notesfalse

\ifnotes
\usepackage{todonotes}

\newcommand{\rnote}[1]{ [\textcolor{red}{Raef: #1}] }
\newcommand{\mnote}[1]{ [\textcolor{purple}{Mike: #1}] }
\newcommand{\cnote}[1]{ [\textcolor{blue}{Cristobal: #1}] }

\else

\newcommand{\rnote}[1]{}
\newcommand{\mnote}[1]{}
\newcommand{\cnote}[1]{}
\fi

\usepackage{tikz}
\def\checkmark{\tikz\fill[scale=0.4](0,.35) -- (.25,0) -- (1,.7) -- (.25,.15) -- cycle;}
\newcommand{\cA}{\mathcal{A}}
\newcommand{\cO}{\ensuremath{\mathcal{O}}}
\newcommand{\cX}{\ensuremath{\mathcal{X}}}
\newcommand{\cY}{\ensuremath{\mathcal{Y}}}
\newcommand{\cZ}{\ensuremath{\mathcal{Z}}}	
\newcommand{\cD}{\mathcal{D}}

\newcommand{\cQ}{\mathcal{Q}}
\newcommand{\cV}{\mathcal{V}}

\newcommand{\cM}{\mathcal{M}}
\newcommand{\cN}{\mathcal{N}}

\newcommand{\cW}{\mathcal{W}}

\newcommand{\cP}{\mathcal{P}}

\newcommand{\bE}{\mathbf{E}}

\newcommand{\re}{\mathbb{R}}
\newcommand{\E}{\mathbb{E}}
\newcommand{\dir}{e}

\newcommand{\grd}{\nabla}
\newcommand{\tgrd}{\widetilde{\nabla}}
\newcommand{\tdel}{\widetilde{\Delta}}
\newcommand{\tkappa}{\widetilde{\kappa}}
\newcommand{\tN}{\widehat{N}}
\newcommand{\tg}{\widetilde{g}}
\newcommand{\tsigma}{\widehat{\sigma}}
\newcommand{\lap}{\mathsf{Lap}}
\newcommand{\polyfw}{\mathsf{polySFW}}
\newcommand{\nsfw}{\mathsf{nSFW}}

\newcommand{\norm}[1]{\left\|#1\right\|}
\newcommand{\dual}[1]{\left\|#1\right\|_*}
\newcommand{\eqnorm}[1]{\left\|#1\right\|_+}
\newcommand{\twonorm}[1]{\left\|#1\right\|_2}

\newcommand{\infnorm}[1]{\left\|#1\right\|_{\infty}}

\newcommand{\ex}[2]{\underset{#1}{\mathbb{E}}\left[ #2 \right]}

\newcommand{\PP}{\mathbb{P}}

\newcommand{\ind}{\mathbf{1}}

\newcommand{\gap}{\mathsf{Gap}}

\newcommand{\bDel}{\overline{\Delta}}

\newcommand{\ip}[2]{\langle #1,#2\rangle}

\newcommand{\ipwx}{\langle w,x \rangle}

\newcommand{\elly}{\ell^{(y)}}
\newcommand{\ellb}{\ell_{\beta}^{(y)} }
\newcommand{\dellb}{\ell_{\beta}^{(y)\prime}}
\newcommand{\prox}{\mathsf{prox}}
\newcommand{\Asc}{{\cal A}_{\mbox{\sc\footnotesize sc}}}
\newcommand{\oracle}{\cO_{\beta,\alpha,R}}
\newcommand{\rank}{\theta}

\DeclareMathOperator*{\argmin}{arg\,min}

\begin{document}
\title{Differentially Private Stochastic Optimization: \\ New Results in Convex and Non-Convex Settings}
\author{%
	Raef Bassily\thanks{Department of Computer Science \& Engineering, Translational Data Analytics Institute (TDAI), The Ohio State University. \texttt{bassily.1@osu.edu}} 
		\and  Crist\'obal Guzm\'an \thanks{Department of Applied Mathematics, University of Twente
        and Institute for Mathematical and Computational Engineering, Pontificia Universidad Cat\'olica de Chile  \texttt{c.guzman@utwente.nl}}
		\and Michael Menart \thanks{Department of Computer Science \& Engineering, The Ohio State University. \texttt{menart.2@osu.edu}}
	}
\date{}	
\maketitle

\begin{abstract}
  We study differentially private stochastic optimization in convex and non-convex settings. For the convex case, we focus on the family of non-smooth generalized linear losses (GLLs). Our algorithm for the $\ell_2$ setting achieves optimal excess population risk in near-linear time, while the best known differentially private algorithms for general convex losses run in super-linear time. Our algorithm for the $\ell_1$ setting has nearly-optimal excess population risk $\tilde{O}\big(\sqrt{\frac{\log{d}}{n\varepsilon}}\big)$, and circumvents the dimension dependent lower bound of \cite{Asi:2021} for general non-smooth convex losses. In the differentially private non-convex setting, we provide several new algorithms for approximating stationary points of the population risk. For the $\ell_1$-case with smooth losses and polyhedral constraint, we provide the first nearly dimension independent rate, $\tilde O\big(\frac{\log^{2/3}{d}}{{(n\varepsilon)^{1/3}}}\big)$ in linear time. For the constrained $\ell_2$-case with smooth losses, we obtain a linear-time algorithm with rate $\tilde O\big(\frac{1}{n^{1/3}}+\frac{d^{1/5}}{(n\varepsilon)^{2/5}}\big)$. 
  Finally, for the $\ell_2$-case we provide the first method  for {\em non-smooth weakly convex} stochastic optimization with rate $\tilde O\big(\frac{1}{n^{1/4}}+\frac{d^{1/6}}{(n\varepsilon)^{1/3}}\big)$ which matches the best existing non-private algorithm when $d= O(\sqrt{n})$. We also extend all our results above for the non-convex $\ell_2$ setting to the $\ell_p$ setting, where $1 < p \leq 2$, with only polylogarithmic (in the dimension) overhead in the rates.

  
\end{abstract}
\section{Introduction}\label{sec:intro}



    
    Stochastic optimization (SO) is a fundamental and pervasive problem in machine learning, statistics and operations research. Here, the goal is to minimize the expectation of a loss function (often referred to as the \emph{population risk}), given only access to a sample of i.i.d. draws from a distribution. When such a sample entails privacy concerns, differential privacy (DP) becomes an important algorithmic desideratum.

    
    Consequently, differentially private stochastic optimization (DP-SO) has been actively investigated for over a decade. Despite major progress in this area, some crucial problems remain with existing methods. One major problem is the lack of linear-time\footnote{In this work, complexity is measured by the number of gradient evaluations, omitting other operations. This is in line with the oracle complexity model in optimization \cite{NY82}.} 
    algorithms for nonsmooth DP-SO (even in the convex case), whereas its non-private counterpart has minimax optimal-risk algorithms which make a single pass over the data \cite{NY82}. A second challenge arises in DP-SCO for non-Euclidean settings; i.e., when the diameter of the feasible set, and Lipschitzness and/or smoothness of losses are measured w.r.t.~a non-Euclidean norm (e.g., $\ell_p$ norm). 
    In particular, in the $\ell_1$-setting there is a stark contrast between the polylogarithmic dependence on the dimension in the risk achievable for the smooth case and the necessary polynomial dependence on the dimension in the non-smooth case \cite{Asi:2021}.
    
    Finally, our understanding of DP-SO in the non-convex case is still quite limited.
    In the non-convex domain,
    there are only a few prior results,  
    all of which have several limitations.
    First, all existing works either assume that the optimization problem is unconstrained or only consider the empirical version of the problem known as differentially private empirical risk minimization (DP-ERM). Obtaining population guarantees based on the empirical risk potentially limits the applicability of the existing methods either in terms of accuracy or in terms of computational efficiency. In particular, all existing methods require super-linear running time w.r.t. the dataset size. Second, most of the existing works consider only the Euclidean setting.\footnote{One exception is \cite{WX:19} who study the $\ell_1$ setting in the context of DP-ERM under a fairly strong assumption (see Related Work section).} Finally, none of the prior works have studied non-convex DP-SO when the loss is non-smooth.  
    
    
    
    The goal of this work is to provide faster and more accurate methods for DP-SO. Some of the settings we investigate are also novel in the DP literature.

\subsection{Our Results}

{\renewcommand{\arraystretch}{2}
\begin{table}
\centering
\begin{tabular}{|c|c|c|c|c|}
\hline 
Loss & $\ell_p$-Setting & Rate & Linear Time? & Thm. \\
\hline 
\hline 
{\renewcommand{\arraystretch}{0.5}\begin{tabular}{c}
      \small{Convex GLL}\\
     \small{(Nonsmooth)}
\end{tabular}}
& $p=1$ & $\sqrt{\frac{\log d}{n\varepsilon}}$ &  & \ref{thm:l1_glm} \\ 
\cline{2-5}
 & $p=2$ & $\frac{1}{\sqrt{n}}+\frac{\sqrt{d}}{n\varepsilon} $ & Nearly & \ref{thm:l2_glm} \\
\hline
\small{Nonconvex Smooth} & $p=1$ & $\frac{\log^{2/3}{d}}{(n\varepsilon)^{1/3}}$ & \checkmark & \ref{thm:convergence-polySFW} \\ 
\cline{2-5}
& $1 < p \leq 2$ &  $\frac{ \kappa^{2/3}}{n^{1/3}}+\kappa^{2/3}\left(\frac{d\tilde{\kappa}}{n^2\varepsilon^2}\right)^{1/5}$ & \checkmark & \ref{thm:privacy-noisySFW} \\
\hline
{\renewcommand{\arraystretch}{0.5}\begin{tabular}{c}
      \small{Weakly Convex}\\
     \small{(Nonsmooth)}
\end{tabular}}
& $1\leq p \leq 2$ & $\frac{\kappa^{5/4}}{n^{1/4}}+\kappa^{4/3}\Big(\frac{d\tilde{\kappa}}{n^2 \varepsilon^2}\Big)^{1/6}$ &  & \ref{thm:nsmth_ncnvx} \\ 
\hline 
\end{tabular}
\caption{Accuracy bounds and running time for our algorithms. Here, $n$ is sample size, $d$ is dimension, $\varepsilon,\delta$ are the privacy parameters, $\kappa=\min\{\frac{1}{p-1},\log{d}\}$ and $\tilde{\kappa} = 1+\log{d}\cdot\ind(p<2)$. We omit the dependence on factors of order $\mbox{polylog}(n,1/\delta)$. Bounds shown for unit $\ell_p$ ball as a feasible set.
}
\label{table:results}
\end{table}}

We enumerate the different settings we investigate in DP-SO, together with our main contributions.
    
\noindent{\bf Convex generalized linear losses.} Our first case of the study is {\em non-smooth} DP-SCO in the case of {\em generalized linear losses} (GLL). This model encompasses a broad class of problems, particularly those which arise in supervised learning, making it a very important particular case. Here, our contributions are two-fold. First, in the $\ell_2$-setting, we provide the first nearly linear-time algorithm that attains the optimal excess risk. The fastest existing methods with similar risk work for general convex losses, but they run in superlinear time w.r.t.~sample size \cite{Asi:2021,KLL:2021}. Our second contribution here is a nearly-dimension independent excess risk bound in the $\ell_1$-setting\footnote{As in all existing works on DP-SO, in the $\ell_1$-setting we also assume the feasible set to be polyhedral.} for convex non-smooth GLL.  
This result circumvents a general DP-SCO excess risk lower bound in the non-smooth $\ell_1$-setting which shows polynomial dependence on the dimension \cite{Asi:2021}, and it matches the minimax risk in the non-private case when $\varepsilon=\Theta(1)$ \cite{Agarwal:2012}. 
    
    Our two contributions for GLL follow the same simple idea. 
    We leverage the GLL structure, namely the fact that these losses are effectively ``one-dimensional,'' to make a fast approximation of the Moreau envelope of the loss \cite{Moreau:1965}. We can then exploit the smoothness of the envelope  to improve algorithmic performance. A similar approach was taken by \cite{BFTT:19}, but their approach suffered from an increase in the running time by a factor of $n^3$ due to the high cost of approximating the gradient of the envelope, which involves solving a high dimensional strongly convex optimization problem at each iteration. 
    In the case of $\ell_2$, we use an existing linear-time algorithm for smooth DP-SCO with optimal excess risk \cite{FKT:2020} combined with our smoothing approach, which results in an $O(n\log{n})$-time algorithm.  
    In the case of $\ell_1$, we use an existing noisy Frank-Wolfe algorithm that attains {\em optimal empirical risk} for {\em smooth losses} \cite{TTZ15a}, together with generalization bounds for GLLs based on Rademacher complexity \cite{shalev2014understanding}. This algorithm is not linear time, and hence it is tempting to instead use a variant of one pass stochastic Frank-Wolfe algorithms, as in \cite{Asi:2021, BGN:2021}. However, the excess risk of these algorithms has a linear dependence on the smoothness constant, which prevents us from obtaining the optimal risk via smoothing. Hence, it is an interesting future direction to improve the running time in the $\ell_1$-setting. 
    
\noindent{\bf Non-convex Smooth Losses.} Next, we move to the setting of smooth non-convex losses, where the goal is to approximate {\em  first-order stationary points}\footnote{Unless otherwise stated, we will refer to first-order stationary points as stationary points.} (see (\ref{eqn:st-gap}) in Section \ref{sec:Preliminaries}). 
This case has attracted significant attention recently, and it brings major theoretical challenges since most tools used to derive optimal excess risk in DP-SCO, such as uniform stability \cite{hardt-recht-singer'16,bassily2020stability} or privacy amplification by iteration \cite{FKT:2020}, no longer apply.
Here, we provide the first linear time private algorithms. In the $\ell_1$-setting, we obtain a nearly-dimension independent rate $O((\log^2 d/[n\varepsilon])^{1/3})$, which to the best of our knowledge is new, even in the non-private case. 
We suspect that our rates for the $\ell_1$-setting are essentially tight for linear-time algorithms (at least when $\varepsilon=\Theta(1)$): in \cite{ACDFSW:2019}, for non-convex smooth SO
in the $\ell_2$-setting, a lower bound $\Omega(1/n^{1/3})$ is proved for minimizing the norm of the gradient via a stochastic gradient oracle. 
In the $\ell_2$-setting (and more generally, for $\ell_p$-setting, where $1\leq p\leq 2$), our stationarity rate (see Table \ref{table:results}) is slightly worse than the state of the art, $O((d/n^2)^{1/4})$ \cite{ZCHWB:20}. However, in \cite{ZCHWB:20}, only the unconstrained case is considered, and the accuracy measure is the norm of the gradient; moreover, the running time is superlinear, $O(n^2\varepsilon/\sqrt{d})$. 
    
    Our workhorse for these results is a recently developed variance-reduced stochastic Frank-Wolfe method \cite{Hassani:2020,zhang2020one}, which has also proved useful in DP-SCO \cite{Asi:2021,BGN:2021}. This method is based on reducing variance through a recursive estimate of the gradient at the current point, leveraging past gradient estimates and the fact that step-sizes are small. Applying this technique in DP is challenging, as we need to carefully schedule the algorithm in rounds (to prevent gradient error accumulation) and to properly tune step-sizes and noise, in order to trade-off accuracy and privacy. 
    
\noindent{\bf Non-convex non-smooth losses.} We conclude with the case of weakly convex non-smooth stochastic optimization, where we devise algorithms to compute {\em close to nearly-stationary points}. Weakly convex functions are a natural and rather common model in some machine learning applications, including convex composite losses, robust phase retrieval, non-smooth trimmed estimation, covariance matrix estimation, sparse dictionary learning, etc.~(see \cite{DG:2019,DD:2019} and references therein). Moreover, this class subsumes smooth non-convex functions. 
To the best of our knowledge, this setting has not been previously addressed in the DP literature. Our algorithm is inspired by the proximally-guided stochastic subgradient method from \cite{DG:2019}, and it is based on approximating proximal steps w.r.t.~the risk function, where each proximal subproblem is solved through an optimal DP-SCO method for strongly convex losses \cite{Asi:2021}. This algorithm works similarly for the $\ell_1$ and $\ell_2$ settings (and, in fact, $\ell_p$ for any $1\leq p\leq 2$), for which we exploit the strong convexity properties of these spaces. Here again, our non-Euclidean extensions seem to be new, even in the non-private case. Our rates for $\ell_2$-setting match the best existing non-private rates, $O(1/n^{1/4})$, in the  
regime $d=O(\sqrt{n})$ (when $\varepsilon=\Theta(1)$). Finally, we observe that our algorithm runs in time $\tilde O(\min\{n^{3/2},n^2\varepsilon/\sqrt{d}\})$.

\subsection{Related Work}
Differentially private convex optimization has been studied extensively for over a decade (see, e.g., \citep{CMS,jain2012differentially,kifer2012private, BST,JTOpt13, TTZ15a, BFTT:19, FKT:2020}). 
Most of the early works in this area focused 
on the empirical risk minimization problem. The first work to derive minimax optimal excess risk in DP-SCO is \cite{BFTT:19}, which has been further improved, in terms of running time (e.g. \cite{FKT:2020, bassily2020stability, KLL:2021}). 
Non-Euclidean settings in DP convex optimization were studied in \cite{jain2012differentially,TTZ15a}. Nearly optimal rates for non-Euclidean DP-SCO were only recently discovered in \cite{Asi:2021,BGN:2021}. \cite{JTOpt13} was one of the first works to focus on the case of private optimization for GLLs, and showed that dimension independent excess risk was possible in $\ell_1$ and $\ell_2$ settings. These results have since been superseded in the $\ell_1$ case by \cite{Asi:2021} and in the $\ell_2$ case by \cite{SSTT:21}.  



In the non-convex case, \cite{zhang_efficient, wang_differentially_2017, WJEG:19} studied smooth unconstrained DP-ERM in the Euclidean setting. 
Smooth unconstrained DP-SO was studied in \cite{wang19c}, where relatively weak guarantees on the excess risk were shown. Convergence to second-order stationary points of the empirical risk was also studied in the same reference under stronger smoothness assumptions. Smooth constrained DP-ERM was studied in \cite{WX:19} in both $\ell_2$ and $\ell_1$ settings. However, their result in the $\ell_1$ setting entails the strong assumption that the loss is smooth w.r.t.~the $\ell_2$ norm. 
The special case of non-convex smooth GLLs was studied in \cite{SSTT:21}, however, their result is limited to the empirical risk (DP-ERM) in the unconstrained setting. The work of \cite{ZCHWB:20} studied DP-SO in the Euclidean setting, and gave convergence guarantees in terms of the population gradient, however, their results are limited to smooth unconstrained optimization.

\section{\label{sec:Preliminaries}Preliminaries}

\paragraph{Normed Spaces.~} 
Let $(\bE,\|\cdot\|)$ be a normed space of dimension $d,$ 
and let $\langle \cdot,\cdot\rangle$ an arbitrary inner product over $\bE$ (not necessarily inducing the norm $\|\cdot\|$). Given $x\in \bE$ and $r>0$, let ${\cal B}_{\|\cdot\|}(x,r)=\{y\in \bE:\|y-x\|\leq r\}$. 
The dual norm over $\bE$ is defined as usual, $\|y\|_{\ast}\triangleq\max_{\|x\|\leq 1} \langle y,x\rangle$. With this definition, $(\bE,\|\cdot\|_{\ast})$ is also a $d$-dimensional normed space. As a main example, consider the case of $\ell_p^d\triangleq(\re^d,\|\cdot\|_p)$, where $1\leq p\leq \infty$ and $\|x\|_p\triangleq\big(\sum_{j\in[d]} |x_j|^p \big)^{1/p}$. As a consequence of the H\"older inequality, one can prove that the dual of $\ell_p^d$ corresponds to $\ell_q^d$, where $1\leq q\leq \infty$ is the conjugate exponent of $p$, determined by $1/p+1/q=1$. 

\paragraph{Differential Privacy \cite{DKMMN06}.~} 
A randomized algorithm $\cA$ is said to be $(\varepsilon,\delta)$ differentially private (abbreviated $(\varepsilon,\delta)$-DP) if for any pair of datasets $S$ and $S'$ differing in one point and any event $\mathcal{E}$ in the range of $\cA$ it holds that 
\[
\PP[\cA(S)\in\mathcal{E}] \leq e^{\varepsilon}\PP[\cA(S')\in \mathcal{E}] + \delta.
\]

\begin{lem}[Advanced composition \cite{DRV10,DR14}]\label{lem:adv_comp}
For any $\varepsilon > 0, \delta \in [0,1),$ and $\delta' \in (0,1)$, the class of $(\varepsilon,\delta)$-differentially private algorithms satisfies $(\varepsilon', k\delta + \delta')$-differential privacy under $k$-fold adaptive composition,
for $\varepsilon'= \varepsilon\sqrt{2k \log(1/\delta')} + k\varepsilon(e^\varepsilon-1)$.
\end{lem}

\paragraph{Stochastic Optimization.~}


In the Stochastic Optimization problem with $(\bE, \|\cdot\|)$-setting, we have a normed space $(\bE, \|\cdot\|)$; a feasible set $\cW\subseteq\bE$ which is closed, convex and with diameter at most $D$ w.r.t.~$\|\cdot\|$; and loss functions $f:\cW\times\cZ\mapsto\re$ are assumed to be $L_0$-Lipschitz w.r.t.~$\|\cdot\|$. Sometimes, we also consider losses which are $L_1$-smooth: i.e., for all $w,v\in{\cal W}$, $\|\nabla f(w)-\nabla f(v)\|_{\ast}\leq L_1\|w-v\|$. In this problem, there is an unknown distribution ${\cal D}$ over a set $\cZ$, and our goal is to minimize a certain accuracy measure that depends on the population risk, defined as $F_{\cal D}(w)=\mathbb{E}_{z\sim {\cal D}}[f(w,z)]$, when only given access to a sample $S=(z_1,...,z_n)\stackrel{i.i.d.}{\sim}\cD$. In Differentially Private Stochastic Optimization (DP-SO) one is concerned with solving this problem under the constraint that the algorithm used is $(\varepsilon,\delta)$-DP w.r.t.~$S$. 

Depending on additional assumptions of the losses, the accuracy measure in DP-SO may vary. 
In the {\em convex case}, the accuracy of a stochastic optimization algorithm is naturally measured by the excess population risk, defined as
$F_{\cD}(w)-\min_{v\in\cW} F_{\cD}(v)$. For the non-convex case, providing guarantees on the excess population risk is often intractable. 

\paragraph{Non-Convex Stochastic Optimization.~}
In the {\em non-convex smooth} case, a common performance measure to use is the {\em stationarity gap} of the population risk, which for $w\in {\cal W}$ is defined as
\begin{align}
\gap_{F_{\cal D}}(w)=\max_{v\in {\cal W}}\langle \nabla F_{\cal D}(w),w-v\rangle. \label{eqn:st-gap}
\end{align}
Note that if the stationarity gap is zero, then $w$ is indeed a stationary point of the risk. 
For the {\em non-convex non-smooth} case, near stationarity (i.e., small stationarity gap) 
is often a stringent concept, as the set of points with small stationarity gap may coincide with the stationary points themselves. Hence, we will consider instead the goal of finding {\em close to nearly-stationary points} \cite{DG:2019,DD:2019}, which we formally introduce in Section~\ref{sec:weakcnvx}. 


\section{Algorithms for Convex Non-smooth Generalized Linear Losses}\label{sec:cnvx_glms}

In this section we consider the case when $f$ is a non-smooth generalized linear loss. 
\begin{defn}[Generalized Linear Loss] \label{def:glm}
 Let $\cX \subset \re^d$ and $\cY \subseteq \re$. We say that $f:\cW\times(\cX\times\cY)\rightarrow\re$
is an $L_0$-Lipschitz, $R$-bounded GLL with respect to norm $\|\cdot\|$ if $\max_{x\in\cX} \|x\|_* \leq R$ and for every $y\in\cY$ there exists a function
$\elly:\re\rightarrow\re$ such that 
$f(w,(x,y))=\elly(\langle x,w\rangle)$ and
$\elly$ is $L_0$-Lipschitz. 
\end{defn}
We will occasionally refer to the $x$ component of a datapoint as the feature vector. Note the GLL definition implies that $f(\cdot,z)$ is $(L_0 R)$-Lipschitz. 
By smoothing the function $f$ through $\ell$, one can obtain a smoothing which is both efficient and invariant to the norm. The first property can be used to attain an optimal rate for DP-SCO in nearly linear time. The later property allows for an essentially optimal, nearly dimension independent rate in the $\ell_1$ setting for {\em non-smooth} GLLs. 

A critical component of the following results is a technique known as Moreau envelope smoothing \cite{Moreau:1965}. Let $\cM$ be a (potentially unbounded) closed interval, $y\in\re$, and $\beta>0$. Consider a function $\elly:\cM\mapsto\re$  as in Definition \ref{def:glm}. The $\beta$-Moreau envelope of $\elly$ is given as 
\[ \ellb(m) \triangleq \min_{u\in\cM}\big[\elly(u)+\frac{\beta}{2}|u-m|^2\big]. \]

Denote the proximal operator with respect to $\elly$ as
\[ \mbox{prox}_{\elly}^{\beta }(m) = \arg\min\limits_{u\in\cM}\big[\elly(u)+\frac{\beta}{2}|u-m|^2\big]. \]

For convex functions, the Moreau envelope satisfies the following properties.

\begin{lem} \label{lem:moreau}
(See \cite{Nes05,Can11}) 
Let $\elly:\mathbf{\cM}\mapsto\mathbb{R}$ be a convex function and $L_0$-Lipschitz. 
Then the following hold:
\begin{itemize}
\item[(a)] $\ellb$ is convex, $2L_0$-Lipschitz and $\beta$-smooth. 
\item[(b)] $\dellb(m)=\beta[m-\mbox{prox}_{\elly}^{\beta}(m)]$.
\item[(c)] $\ellb(m) \leq \elly(m) \leq \ellb(m)+L_0 ^2/(2\beta)$.
\end{itemize}
\end{lem}

\subsection{Smoothing Generalized Linear Losses}  \label{subsec:smoothing_glm}

\begin{algorithm}[!h]
\caption{$\oracle$: Gradient Oracle for Smoothed GLL}
\begin{algorithmic}[1]

\REQUIRE Parameter Vector $w\in\cW$, Datapoint $(x,y)\in(\cX\times\cY)$

\STATE $m = \ipwx$

\STATE Let $[a,b] = \cM \cap \left[m - \frac{2L_0}{\beta}, m + \frac{2L_0}{\beta}\right]$

\STATE $T=\left\lceil \log_2\left(\frac{16L_0^2R^2}{\alpha^2}\right) \right\rceil $

\FOR{$t=1$ to $T$}
\STATE Let~$m_{t}=\frac{a+b}{2}$

\IF{$\elly(\frac{a+m_{t}}{2}) + |\frac{a+m_{t}}{2} - m |^2 \geq \elly(\frac{m_{t}+b}{2}) + |\frac{m_{t}+b}{2} - m|^2$}
\STATE $b=m_{t}$
\ELSE
\STATE $a=m_{t}$
\ENDIF
\ENDFOR

\STATE $\bar{u}=\argmin\limits_{\{m_t:t\in[T]\}}\{\elly(m_{t}) + |m_{t} - m|^2\}$
\STATE Output: $\beta(m-\bar{u})x$
\end{algorithmic}
\label{Alg:gll_oracle}
\end{algorithm}

Existing works such as \cite{BFTT:19} have used the Moreau envelope smoothing for DP-SCO, but suffer from the high computational cost of computing the proximal operator. For GLLs, we can smooth $\ell$ instead of $f$ to obtain a smoothed function efficiently. We have the following guarantee for the smoothed version of $f$.

\begin{lem} \label{lem:f_smoothness}
Let $(x,y)\in(\cX\times\cY)$. Let $\ellb$ be the Moreau envelope of $\elly$ and define
$f_{\beta}(w,(x,y)) = \ellb(\langle w,x \rangle)$. Then $f_\beta$ is $2L_0R$-Lipschitz and $\beta\|x\|_{*}^2$-smooth with respect to $\|\cdot\|$ and 
$|f(w,(x,y))-f_{\beta}(w,(x,y))| \leq \frac{2L_0^2}{\beta}$ 
for all $w\in\cW$.
\end{lem}

By smoothing $f$ through $\ell$, we reduce the evaluation of the proximal operator to a {\em $1$-dimensional convex problem}. This allows us to use the bisection method to obtain the following oracle for $f_\beta$ which runs in logarithmic time.
\begin{lem} \label{lem:oracle}
Let $\beta,\alpha>0$ and let $\norm{\cdot}$ be a norm.
Then the there exists a gradient oracle, $\oracle$~ for $f_\beta$ (Algorithm \ref{Alg:gll_oracle})
which satisfies
$\dual{\nabla f_{\beta}(w,(x,y)) - \oracle(w,(x,y))} \leq \alpha$ for any $x$ such that $\dual{x} \leq R$. Further, $\oracle$~ has running time 
$O\left(\log(L_0^2 R^2/\alpha^2)\right)$. \mnote{Term inside log was $L_0 R \beta/ \alpha^2$ but should have been $L_0^2 R^2 / \alpha^2$}
\end{lem}
\begin{proof}
Let $x,y$ and $w$ be the inputs to Algorithm \ref{Alg:gll_oracle}. Note as defined in Algorithm \ref{Alg:gll_oracle}, $m=\ipwx$ and
$\cP=\cM \cap \big[m - \frac{2L_0}{\beta}, m + \frac{2L_0}{\beta}\big]$. 
Define $h_{\beta}(u) \triangleq \elly(u) + \frac{\beta}{2}|u-m|^2$, i.e. the proximal loss. 
Let $u^{*}=\argmin\limits_{u\in\re}\{h_{\beta}(u)\}$. 
We first show that $|\bar{u}-u^*|$ is small by noting that lines 1-10 of Algorithm \ref{Alg:gll_oracle} implement the bisection method on $h_\beta$ (see, e.g., \cite[ Theorem~1.1.1]{Nemirovski95_notes}). Thus, so long as $\cP$ is a closed interval, $u^*\in\cP$, and $\max\limits_{u\in\cP}\{h_\beta(u)-h_\beta(u^*)\}\leq\tau$, standard guarantees of the bisection method give that 
$h_{\beta}(\bar{u})-h_{\beta}(u^{*})\leq\tau 2^{-T}$.
Clearly $\cP$ is a closed interval since $\cM$ is closed.
To see that $u^*\in \cP$, note that since $u^*$ is the minimizer of $h_{\beta}$ it holds that 
\[
0 \leq \elly(m) + \frac{\beta}{2}|m-m|^2 - \elly(u^*) - \frac{\beta}{2}|u^*-m|^2 
   = \elly(m) - \elly(u^*) - \frac{\beta}{2}|u^*-m|^2. 
\]
Further since $\elly$ is $L_0$-Lipschitz we have that $\elly(m)-\elly(u^*) \leq L_0|u^*-m|$. Using this fact in the above inequality we obtain
$|m-u^*| \leq 2L_0/\beta$ and thus $u^* \in \cP$.
Using the bound on the radius of $\cP$ and Lipschitz constant of $\elly$ it holds that 
$\tau\leq 8L_0^2/\beta$.
The setting of $T=\left\lceil \log_2\left(\frac{16L_0^2R^2}{\alpha^2}\right) \right\rceil$ and the accuracy gaurantees of the bisection method then gives that
$h_{\beta}(\bar{u})-h_{\beta}(u^{*})\leq\frac{\alpha^{2}}{2\beta R^2}$.
Since $h_{\beta}$ is $\beta$-strongly convex we then have
\[ 
|\bar{u}-u^{*}|  \leq\sqrt{\frac{2\left(h_{\beta}(\bar{u})-h_{\beta}(u^{*})\right)}{\beta}} 
 \leq \frac{\alpha}{\beta R}.
\]
The accuracy guarantee $\dual{\oracle(w,(x,y)) - \nabla f_\beta(w,(x,y))} \leq\alpha$ then follows straightforwardly using part (b) of Lemma \ref{lem:moreau} and the facts that $\dual{x}\leq R$ and $u^*=\mbox{prox}_{\elly}^{\beta}(m)$.
\end{proof}

\mnote{Removed "New results from smoothing" header and made paragraph sections the subsections}
\subsection{Linear Time DP-SCO in the $\ell_2$ Setting~}

\begin{algorithm}[!h]
\caption{Phased SGD for GLL}
\begin{algorithmic}[1]

\REQUIRE Private dataset $\big(z_1,\dots,z_n\big) \in (\cX\times\cY)^n$, constraint set $\cW\subseteq\re^d$, privacy parameters $(\varepsilon, \delta)$ s.t. $\varepsilon\leq\sqrt{\log(1/\delta)}$, constraint diameter (for constrained case) $D$, Lipschitz constant $L_0$, 
smoothness parameter $\beta$, oracle accuracy $\alpha$, feature vector norm bound $R$

\STATE Let $\tilde{w}_0\in\cW$ be arbitrary

\STATE $\rho=\frac{\varepsilon}{2\sqrt{\log(1/\delta)}}$

\STATE  $K=\log_2(n)$

\STATE  For \textbf{Constrained} setting: $\eta=\frac{D}{3L_0R}\min\{\frac{\rho}{\sqrt{d}},\frac{1}{\sqrt{n}}\}$

\STATE  For \textbf{Unconstrained} setting: $\eta=\frac{1}{3L_0R}\min\{\frac{\rho}{\sqrt{\rank}},\frac{1}{\sqrt{n}}\}$, where $\rank$ is an upper bound on the \emph{expected rank} of $\sum_{i=1}^n x_i x_i^{\top}$. (Note that we always have $\rank\leq n$.)

\STATE  $s=1$

\FOR{$k=1$ to $K$}
    \STATE  $T_{k}=\frac{n}{2^{k}}$
    
    \STATE  $\eta_{k}=\frac{\eta}{4^{k}}$ 
    
    \STATE Initialize PSGD algorithm of \cite{FKT:2020} (over domain $\cW$) at $\tilde{w}_{k-1}$ and run with oracle $\oracle$ in place of $\nabla f$ and step size $\eta_k$ for $T_k$ steps over dataset $\{z_s,...,z_{s+T_k}\}$. Let $w_k$ be the average of the iterate of PSGD.
    
    \STATE  $\tilde{w}_{k}=w_{k}+\xi_{k}$~where~$\xi_{k}\sim\mathcal{N}(0,\mathbb{I}_{d}\sigma_{k}^{2})$~with~$\sigma_{k}=\frac{4L_0R\eta_{k}}{\rho}$
    
    \STATE  $s=s+T_{k}$
\ENDFOR

\STATE Output: $\tilde{w}_{K}$
\end{algorithmic}
\label{Alg:phased}
\end{algorithm}

Given the oracle described in Algorithm \ref{Alg:gll_oracle}, we can optimize $f_{\beta}$ using the linear time Phased-SGD algorithm of \cite{FKT:2020}. When using $\oracle$ instead of the true gradient oracle, $\nabla f$, we need account for two additive penalties, the increase in error due to using the approximate gradient and the increase in error to due to minimizing the smoothed function. We ultimately have the following guarantee. 

\begin{thm} \label{thm:l2_glm}
Let $\mathcal{W}\subset\re^d$ have $\|\cdot\|_2$-diameter at most $D$.
Let $f:\cW\times(\cX\times\cY)\rightarrow\re$
be a $L_0$-Lipschitz and $R$-bounded GLL with respect to $\|\cdot\|_2$.
Let $\beta=\sqrt{n}L_0/R$, $\alpha=\frac{L_0 R}{n\log{n}}$. Then Phased-SGD run with 
oracle $\oracle$ and dataset $S\in(\cX\times\cY)^n$ satisfies $(\varepsilon,\delta)$ differential privacy and has running time $O(n\log{n})$. Further, if $S\sim\cD^n$ the output of Phased-SGD has expected excess population risk 
$O\left(L_0R D \left(\frac{\sqrt{d\log(1/\delta)}}{n\varepsilon}+\frac{1}{\sqrt{n}}\right)\right)$.
\end{thm}
\begin{proof}
The proof follows similarly to \cite{FKT:2020}, but additionally we account for the change in gradient sensitivity and extra error introduced by using the approximate gradient oracle of the smoothed loss, $\oracle$. 
Let PSGD($\cO,\eta,w_0,T$) (used in Algorithm \ref{Alg:phased}) denote the process which computes $w_{t}=\Pi_\cW[w_{t-1}+\eta\cO(w_{t-1})]:\forall t\in[T]$, where $\Pi_\cW$ is the projection onto constraint set $\cW$.
By Lemma~\ref{lem:f_smoothness}, $f_{\beta}$ is a $(2L_0 R)$-Lipschitz and $(\beta R^2)$-smooth loss function. 
Further, the increase in error due to using $\alpha$-approximate gradients in SGD is at most $2\alpha D$ (see, e.g., \cite{FGV:16, BFTT:19}). 
Let $F_{\beta,\cD}(w) = \ex{z\sim\cD}{f_\beta(w)}$ and let $w_{\beta}^* = \argmin\limits_{w\in\cW}\{F_{\beta,\cD}(w)\}$. For notational convenience, let $w_0=w_{\beta}^*$ and $\sigma_0=D$. 
We have (following from \cite[Lemma~4.5 \& Proof of Theorem~4.4]{FKT:2020}): \cnote{What is $T$ below?} \mnote{Should be n. I believe there should also be a $L^2R^2\eta$ term.} \mnote{Actually, I think the better way to handle this is to not pull the first term out and define $\sigma_0 = D$ (instead of defining $\xi_0$, as we don't actually use $\xi_0$ in this proof. We can still use $\xi_0$ in the unconstrained proof} \cnote{Ideally, we should phrase it in a way where this follows directly from some bound in FKT.}\rnote{It's ok to make it more self contained as long as we cite FKT, which we do.}

\begin{align*}
\ex{}{F_{\beta,\cD}(w_K) - F_{\beta,\cD}(w_{\beta}^*)} & =
\sum_{k=1}^{K}\ex{}{F_{\beta,\cD}(w_k) - F_{\beta,\cD}(w_{k-1})} + \ex{}{F_{\beta,\cD}(\tilde{w}_K) - F_{\beta,\cD}(w_K)} \\
& \leq \sum_{k=1}^{K} \left( \frac{d\sigma_{k-1}^2}{2\eta_k T_k} + 2\eta_k L_0^2 R^2 + 2D\alpha \right) + 2L_0 R\, \E[\|\xi_K\|_2]  \\
& = \sum_{k=2}^{K} \left( \frac{d\sigma_{k-1}^2}{2\eta_k T_k} + 2\eta_k L_0^2 R^2 \right) + 2L_0 R \sqrt{d}\sigma_K + 2DK\alpha. 
\end{align*}
By the setting of $\alpha=\frac{L_0 R}{n\log(n)}$, we have $2DK\alpha = \frac{2 L_0 R D}{n}$. 
It can be verified that the rest of the expression is 
$O \left(L_0 R D \left( \frac{1}{\sqrt{n}} + \frac{\sqrt{d}}{\rho n} \right) \right)$ 
(see \cite[Proof~of~Theorem 4.4]{FKT:2020}). 
To convert to population loss with respect to the original function, we provide the following analysis. Let $w^* =\min\limits_{w\in\cW}F_{\cD}(w^*)$. By Lemma \ref{lem:f_smoothness} we have for any $w\in\cW$
\begin{align*}
F_{\cD}(w) - F_{\cD}(w^*) & \leq F_{\beta,\cD}(w) - F_{\beta,\cD}(w^*) + \frac{L_0^2}{\beta} \\
& \leq F_{\beta,\cD}(w) - F_{\beta,\cD}(w_{\beta}^*) + \frac{L_0^2}{\beta}.
\end{align*}

Thus by the setting $\beta=\sqrt{n}L_0/(RD) $ 
we have
\[\ex{}{F_\cD(\tilde{w}_K) - F_\cD(w^*)} = O\left(L_0 R D \left(\frac{1}{\sqrt{n}} + \frac{\sqrt{d}}{\rho} \right)\right). \]
Plugging in our value of $\rho$ into the above we have the final result.
\[
\ex{}{F_{\cD}(\tilde{w}_K)-F_{\cD}(w^{*})} = O\left(L_0 R D \left(\frac{1}{\sqrt{n}}+\frac{\sqrt{d\log(1/\delta)}}{n\varepsilon}\right)\right).
\]

For privacy, note that
$\|\oracle(w,z)\| \leq (2L_0R + \frac{L_0R}{n})$, 
and thus the sensitivity of the approximate gradient is bounded by $3L_0R$.
Thus, by setting the parameters of Phased SGD as they would be for a $(3L_0R)$-Lipschitz function, Lemma 4.5 of \cite{FKT:2020} implies that Algorithm \ref{Alg:phased} satisfies ($\varepsilon,\delta)$-DP so long as $\eta \leq \frac{2}{\beta R^2}$. It's easy to see that the condition on $\eta$ holds. 
\end{proof}

Furthermore, it is possible to adapt this technique to the unconstrained case as well ($\cW = \re^d$). It was shown in \cite{SSTT:21} that in the unconstrained case, dimension independent rates are attainable. The following theorem establishes that such rates are achievable in near linear time (as opposed to the super linear rates of \cite{SSTT:21}). 

Before stating the theorem, a few preliminaries are necessary. Let $V$ be a matrix whose columns are an eigenbasis for $\sum_{i=1}^n x_i x_i^{\top}$.
For any $u,u'\in\re^d$, let $\|u\|_V = \sqrt{u^{\top} V V^{\top} u}$ denote the semi-norm of $u$ induced by $V$, and let $\ip{u}{u'}_V=u^{\top} V V^T u'$. Here, we assume knowledge of some upper bound $\rank$ on $\ex{S\sim\cD}{\mathsf{Rank}(V)}$. Note that this is no loss of generality since we always have $\ex{S\sim\cD}{\mathsf{Rank}(V)}\leq n$; hence, if we don't have this additional knowledge, we can set $\rank = n$. 

\begin{thm} \label{thm:l2_glm_unconstrained}
Let $\mathcal{W}=\re^d$.
Let $f:\cW\times(\cX\times\cY)\rightarrow\re$
be a $L_0$-Lipschitz and $R$-bounded GLL with respect to $\|\cdot\|_2$.
Let $\beta=\sqrt{n}L_0/R$, $\alpha=\frac{L_0 R}{n\log{n}}$. Let $\rank$ be as defined above. Then Phased-SGD run with 
oracle $\oracle$ and dataset $S\in(\cX\times\cY)^n$ satisfies $(\varepsilon,\delta)$ differential privacy and has running time $O(n\log{n})$. Further, if $S\sim\cD^n$ the output of Phased-SGD has expected excess population risk 
$O\left(L_0 R\left(\|\tilde{w}_0 - w_\beta^*\|^2+1\right)\left(\frac{\sqrt{\rank\log(1/\delta)}}{n\varepsilon} + \frac{1}{\sqrt{n}}\right)\right)$ 
\end{thm}
To prove the theorem, we start by providing the following lemma. 
As before, denote $w_0=w_{\beta}^*$ and define $\xi_0=\tilde{w}_0-w_\beta^*$. \mnote{Added "define" to indicate that $\xi_0$ is new. See pg 8}   
\begin{lem}\label{lem:phase_bound}
Let $\alpha,\beta,R$ be as in Theorem \ref{thm:l2_glm_unconstrained}. Then the output, $w_k$, of phase $k$ of Phased SGD using $\oracle$ satisfies 
\[
\ex{}{F_{\beta,D}(w_k)-F_{\beta,D}(w_{k-1})} \leq \frac{\ex{}{\|\tilde{w}_{k-1} - w_{k-1}\|_V^2}}{2\eta_k T_k} + \frac{5\eta_k L_0^2R^2}{2} + \frac{L_0 R \left(\ex{}{\|\tilde{w}_{k-1} - w_{k-1}\|_V} + 1\right)}{\sqrt{n}\log(n)}.
\]
\end{lem}
\begin{proof}
Let $\{u_0,\dots,u_{T_k}\}$ denote the iterates generated by round $k$ of PSGD (where $u_0=\tilde{w}_{k-1}$), and let $z_t$ be the datapoint sampled during iteration $t$. For all $t\in\{0,...T_k\}$, define the potential function $\Phi^{(t)} \triangleq \|u_t - w_{k-1}\|_V^2$. Using standard algebraic manipulation, we have 
\begin{align*}
 \Phi^{(t+1)} & = \Phi^{(t)} - 2\eta_k \ip{\oracle(u_t,z_t)}{u_t-w_{k-1}}_V + \eta^2_k\|\oracle(u_t,z_t)\|^2_V \\
 & \leq \Phi^{(t)} - 2\eta_k \ip{\nabla f_\beta(u_t,z_t)}{u_t-w_{k-1}}_V + 2\eta_k \alpha \|u_t - w_{k-1}\|_V + \eta^2_k (\alpha^2 + 4L_0^2R^2),
\end{align*}
where the inequality follows from the fact that 
$\|\oracle(u_t,z_t) - \nabla f_{\beta}(u_t, z_t) \| \leq \alpha$
and the nonexpansiveness of the projection onto the span of $V$.
Since the gradient is in the span of $V$, we have
\[\Phi^{(t+1)} \leq \Phi^{(t)} - 2\eta_k \ip{\nabla f_\beta(u_t,z_t)}{u_t-w_{k-1}} + 2\eta_k \alpha \|u_t - w_{k-1}\|_V + \eta^2_k (\alpha^2 + 4L_0^2R^2).\]
Hence
\[
\ip{\nabla f_\beta(u_t,z_t)}{u_t - w_{k-1}} \leq \frac{\Phi^{(t)} - \Phi^{(t+1)}}{2\eta_k} + \alpha \|u_t - w_{k-1}\|_V + \frac{\eta_k}{2}(\alpha^2 + 4L_0^2 R^2).
\]
Taking the expectation w.r.t. all randomness (i.e., w.r.t. $S\sim \cD^n$ and the Gaussian noise random variables), we have
\[
\ex{}{\ip{\nabla F_{\beta,D}(u_t)}{u_t - w_{k-1}}} \leq \frac{\ex{}{\Phi^{(t)} - \Phi^{(t+1)}}}{2\eta_k} + \alpha \ex{}{\|u_t - w_{k-1}\|_V} + \frac{\eta_k}{2}(\alpha^2 + 4L_0^2 R^2).
\]
Moreover, by the convexity of $F_{\beta,D}$ we have 
$\ex{}{\ip{\nabla F_{\beta,D} (u_t)}{u_t - w_{k-1}}} \geq \ex{}{F_{\beta,D}(u_t) - F_{\beta,D}(w_{k-1})}$.
Combining this inequality with the above, and using the fact that 
$w_k = \frac{1}{T_k}\sum_{t=1}^{T_k} u_t$ together with the convexity of $F_{\beta,D}$, 
we have
\begin{align*}
\ex{}{F_{\beta,D}(w_k) - F_{\beta,D}(w_{k-1})} 
& \leq \frac{1}{T_k}\sum_{t=1}^{T_k}\left(\ex{}{F_{\beta,D}(u_t) - F_{\beta,D}(w_{k-1})}\right) \\
& \leq \frac{\ex{}{\Phi^{(0)}}}{2\eta_k T_k} + \frac{\alpha}{T_k}\ex{}{\sum_{t=1}^{T_k} \|u_t - w_{k-1}\|_V} + \frac{\eta_k}{2}(\alpha^2 + 4L_0^2 R^2).
\end{align*}

To bound $\ex{}{\sum_{t=1}^{T_k} \|u_t - w_{k-1}\|_V}$ in the above, observe that,
\begin{align*}
\|u_t - w_{k-1}\|_V & \leq \|u_{t-1} - w_{k-1}\|_V + \|u_t - u_{t-1}\|_V \\
& \vdots \\
& \leq \|\tilde{w}_{k-1} - w_{k-1}\|_V + \sum_{j=1}^t\|u_{j} - u_{j-1}\|_V. 
\end{align*}
Hence
\begin{align*}
\ex{}{\|u_t - w_{k-1}\|_V} & \leq \ex{}{\|\tilde{w}_{k-1} - w_{k-1}\|_V} + \sum_{j=1}^t \ex{}{\|u_j - u_{j-1}\|_V} \\
& \leq \ex{}{\sqrt{\Phi^{(0)}}} + \eta_k t (2L_0 R + \alpha),
\end{align*}
where the last inequality follows from the definition of $\Phi^{(0)}$ and the fact that 
$$\ex{}{\|u_j - u_{j-1}\|_V} = \eta_k\ex{}{\|\oracle(u_{j-1}, z_{j-1})\|}\leq \eta_k(2L_0 R + \alpha).$$
Thus we have
\begin{align*}
\ex{}{F_{\beta,D}(w_k) - F_{\beta,D}(w_{k-1})} 
& \leq \frac{\ex{}{\Phi^{(0)}}}{2\eta_k T_k} + \alpha \left(\ex{}{\sqrt{\Phi^{(0)}}} + T_k \eta_k(2L_0 R + \alpha)\right) + \frac{\eta_k}{2}(\alpha^2 + 4L_0^2R^2)\\
& = \frac{\ex{}{\Phi^{(0)}}}{2\eta_k T_k} + \frac{5 \eta_k L_0^2 R^2}{2} + \alpha \left(\ex{}{\sqrt{\Phi^{(0)}}} + 3T_k \eta_kL_0 R \right). 
\end{align*}
The last step follows from the fact that $\alpha=\frac{L_0 R}{n\log(n)} \leq L_0R$.
Further, since 
$\eta_k = \frac{1}{3L_0R_0}\min\{\frac{\rho}{\sqrt{\rank}},\frac{1}{\sqrt{n}}\} \leq \frac{1}{3L_0R\sqrt{n}}$ and $T_k\leq n$
it holds that $3T_k\eta_kL_0R \leq \sqrt{n}$.
Thus by the setting of $\alpha$, we have
\begin{align*}
\ex{}{F_{\beta,D}(w_k) - F_{\beta,D}(w_{k-1})} \leq \frac{\ex{}{\Phi^{(0)}}}{2\eta_k T_k} + \frac{5\eta_k L_0^2 R^2}{2} + \frac{L_0 R\left(\ex{}{\sqrt{\Phi^{(0)}}} + 1\right)}{\sqrt{n}\log(n)}.
\end{align*}
\end{proof}

With this result established, we can now prove Theorem \ref{thm:l2_glm_unconstrained}.
\paragraph{Proof of Theorem \ref{thm:l2_glm_unconstrained}}
Recall that we denote $w_0=w_{\beta}^*$ and $\xi_0=\tilde{w}_0-w_\beta^*$. Using the above lemma and noting that $\tilde{w}_{k-1} - w_{k-1}=\xi_{k-1}$, the excess risk of the $\tilde{w}_K$ is bounded by 
\begin{align}
\ex{}{F_{\beta,\cD}(\tilde{w}_K) - F_{\beta,\cD}(w_{\beta}^*)} 
& = \sum_{k=1}^{K}\ex{}{F_{\beta,\cD}(w_k) - F_{\beta,\cD}(w_{k-1})} +
\ex{}{F_{\beta,\cD}(\tilde{w}_K) - F_{\beta,\cD}(w_K)} \nonumber \\
& \leq \sum_{k=1}^{K} \left( \frac{\ex{}{\|\xi_{k-1}\|_V^2}}{2\eta_k T_k} +
\frac{5\eta_k L_0^2 R^2}{2} + \frac{L_0 R\left(\ex{}{\|\xi_{k-1}\|_V}+1\right)}{\sqrt{n}\log(n)} \right) \nonumber \\
& \quad + \ex{}{F_{\beta,\cD}(\tilde{w}_K) - F_{\beta,\cD}(w_K)}.  \label{eq:un_ex_risk}
\end{align}
Note that for any $2\leq k\leq K$, we have
\begin{align*}
\ex{}{\|\xi_{k-1}\|_V^2} & = \ex{V}{\ex{\xi_{k-1}}{\xi_{k-1}^{\top}VV^{\top}\xi_{k-1}|V}} \leq \ex{V}{\mathsf{Rank}(V)}\sigma_{k-1}^2\leq \rank\sigma_{k-1}^2
\end{align*}
At round $k=1$, we simply have $\ex{V}{\|\xi_0\|_V} \leq \|\tilde{w}_0 - w_\beta^*\|$.
Finally, since $f$ is a GLL, the expected increase in loss due to $\xi_K$ is bounded as  
\begin{align*}
 \ex{}{F_{\beta,D}(\tilde{w}_K) - F_{\beta,D}(w_K)} 
 & = \ex{(x,y)\sim \cD}{\ex{\xi_K}{\ellb(\ip{\tilde{w}_K}{x}) - \ellb(\ip{w_K}{x})}} \\
 & \leq \ex{(x,y)\sim \cD}{\ex{\xi_K}{L_0 |\ip{\xi_K}{x}|}} \\
 & \leq L_0 R \sigma_K  \\
 & = \frac{L_0 R}{4^{K-1}\sqrt{n}} \\
 & = \frac{L_0R}{4n^{5/2}}
\end{align*}
The second inequality follows from the fact that $\ellb$ is $L_0$-Lipschitz, and the last two steps follow form the fact that $\sigma_k \leq \frac{1}{4^{k-1}\sqrt{n}}$ and $K=\log_2(n)$.
Thus, using inequality \eqref{eq:un_ex_risk} above, we have

\begin{align*}
\ex{}{F_{\beta,\cD}(\tilde{w}_K) - F_{\beta,\cD}(w_{\beta}^*)} 
& = O\left(L_0 R\left(\|\tilde{w}_0 - w_\beta^*\|^2+1\right)\left(\frac{\sqrt{\rank}}{n\rho} + \frac{1}{\sqrt{n}}\right)\right) + \\
& \quad \sum_{k=2}^{K} \left( \frac{\rank\sigma_{k-1}^2}{2\eta_k T_k} +
\frac{5\eta_k L_0^2 R^2}{2} + \frac{L_0 R(\sqrt{\rank}\sigma_{k-1}+1)}{\sqrt{n}\log(n)} \right) 
 + \frac{L_0R}{4n^{5/2}}\\
& = O\left(L_0 R\left(\|\tilde{w}_0 - w_\beta^*\|^2+1\right)\left(\frac{\sqrt{\rank}}{n\rho} + \frac{1}{\sqrt{n}}\right)\right) + 
\sum_{k=2}^{K} \left( \frac{\rank\sigma_{k-1}^2}{2\eta_k T_k} +
\frac{5\eta_k L_0^2 R^2}{2} \right) + \frac{3L_0 R}{\sqrt{n}} \\
& = O\left(L_0 R\left(\|\tilde{w}_0 - w_\beta^*\|^2+1\right)\left(\frac{\sqrt{\rank}}{n\rho} + \frac{1}{\sqrt{n}}\right)\right) +
O\left( L_0R\left(\frac{\sqrt{\rank}}{n\rho} + \frac{1}{\sqrt{n}}\right) \right) \\
& = O\left(L_0 R\left(\|\tilde{w}_0 - w_\beta^*\|^2+1\right)\left(\frac{\sqrt{\rank}}{n\rho} + \frac{1}{\sqrt{n}}\right)\right).
\end{align*}

The first line comes from bounding the term corresponding to $k=1$ in the sum in (\ref{eq:un_ex_risk}), and the settings of $\eta_1 = \frac{\rho}{12L_0R\sqrt{n}}$ and $T_1 = n/2$. The second equality follows from the fact that $\sqrt{\rank}\sigma_{k-1} = 4\sqrt{\rank}L_0R\,\eta_{k-1}/\rho \leq  4\sqrt{\rank}L_0R\,\eta/\rho
\leq 2$, and the fact that $K=\log_2(n)$. The third step follows from the choices of $\eta_k,T_k$ and $\sigma_{k-1}$. 
To reach the final result, we convert the guarantee above to a guarantee for the original (unsmoothed) loss and use the setting of $\beta=\sqrt{n}L_0/R$ (as done in the proof of Theorem~\ref{thm:l2_glm}).

\subsection{Better Rate in the $\ell_{1}$ Setting~}
Another interesting consequence of the smoothing method described in section \ref{subsec:smoothing_glm} is that, because it is scalar in nature, it allows one to achieve better rates in the $\ell_1$-setting. In \cite{Asi:2021} it was shown that the optimal rate for general non-smooth losses under $(\varepsilon,\delta)$-DP was roughly
$\Omega(\sqrt{d}/[n\varepsilon\log{d}])$. 
However, their lower bound does not apply to GLLs.
In the following, we show that using the smoothing technique previously described we can achieve a better rate of $\tilde{O}(1/\sqrt{n\varepsilon})$.
We note this rate is optimal in the regime $\varepsilon=\Theta(1)$ \cite{Agarwal:2012}.

\begin{algorithm}[!h]
    \caption{Noisy Frank Wolfe}
    \begin{algorithmic}[1]
    
    \REQUIRE Private dataset $S=(z_1,...,z_n)\in(\cX\times\cY)^n$, polyhedral set $\cW$ with vertices $\cV$, Lipschitz constant $L_0$, constraint diameter $D$, privacy parameters $(\varepsilon, \delta)$,
    smoothness parameter $\beta$, oracle accuracy $\alpha$, feature vector norm bound $R$
    
    \STATE Let $w_1\in\cW$ be arbitrary
    \STATE $T= \frac{n\varepsilon}{\log(|V|)\log(n)\sqrt{\log(1/\delta)}}$
    \STATE $s = \frac{3 L_0 R D \sqrt{8T\log(1/\delta)}}{{n\varepsilon}}$
    
    \FOR{$t=1$ to $T$}
    \STATE $\tilde{\nabla}_t = \frac{1}{n}\sum_{z\in S}\oracle(w_t,z)$
    \STATE Draw $\{b_{v,t}\}_{v\in \cV}$ i.i.d from $\lap(s)$
    \STATE $\tilde{v_t} = \argmin\limits_{v\in \cV} \{\langle v, \tilde{\nabla}_t \rangle + b_{v,t}\}$
    \STATE $w_{t+1} = (1-\mu_t)w_t + \mu_t\tilde{v_t}$, where $\mu_t=\frac{3}{t+2}$
    \ENDFOR
    
    \STATE Output: $w_T$
    \end{algorithmic}
    \label{Alg:nfw}
\end{algorithm}

\begin{thm} \label{thm:l1_glm}
Let $\mathcal{W}\subset\re^d$ be a polytope defined by a set of vertices $\cV$ of cardinality $J$, where $\cW=Conv(\cV)$ and $\cW$ has $\|\cdot\|_1$-diameter at most $D$.  
Let $f:\cW\times(\cX\times\cY)\rightarrow\re$
be a $L_0$-Lipschitz and $R$-bounded GLL with respect to $\|\cdot\|_1$.
Let $\beta=\frac{L_0 \sqrt{n\varepsilon}}{RD\log^{1/4}(1/\delta)\sqrt{\log(J)\log(n)}}$ and
$\alpha=\frac{1}{n\log(n)}$. 
Then Noisy Frank Wolfe (Algorithm \ref{Alg:nfw}) 
with oracle 
$\oracle$
and dataset $S\in(\cX\times\cY)^n$ satisfies $(\varepsilon,\delta)$-differential privacy. Further, if $S\sim\cD^n$ the output of Noisy Frank Wolfe has expected excess population risk 
$O\left(L_0 R D \left( \frac{\log^{1/4}(1/\delta)\sqrt{\log(J)\log(n)}}{\sqrt{n\varepsilon}} + \frac{\sqrt{\log{d}}}{\sqrt{n}} \right)\right)$. 
\end{thm}
\begin{proof}
The proof follows from the analysis of noisy Frank Wolfe from \cite{talwar_private_2016}. Let
$F_{\beta,S}(w) = \frac{1}{n}\sum_{z\in S} f_{\beta}(w,z)$.
Define $w_{\beta,S}^*$ as the minimizer $F_{\beta,S}$ in $\cW$.

Define
$\gamma_t = {\langle \tilde{v_t}, \tilde{\nabla}_t \rangle - \min\limits_{v\in\cV}\langle v,\tilde{\nabla}_t \rangle}.$ Since $F_{\beta,S}$ is $(\beta R^2)$-smooth (by Lemma \ref{lem:f_smoothness}), standard analysis of the Noisy Frank-Wolfe algorithm yields (see, e.g., \cite{TTZ15a}) 
\begin{align*}
\ex{}{F_{\beta,S}(w_T) - F_{\beta,S}(w_{\beta,S}^*)} &\leq O\left(\frac{\beta R^2 D^2}{T}\right) + D\sum_{t=1}^T \mu_t\ex{}{\infnorm{\tilde{\nabla}_t - \nabla F_{\beta,S}(w_t)}} + \sum_{t=1}^T \mu_t\ex{}{\gamma_t}.
\end{align*}

By a standard argument concerning the maximum of a collection of Laplace random variables, we have for all $t\in[T]$
$\ex{}{\gamma_t} \leq 2s\log(|\cV|)$.
Note also that for all $t$, by the approximation guarantee of $\oracle$, we have (with probability 1) 
$\infnorm{\tilde{\nabla}_t - \nabla F_{\beta,S}(w_t)} \ \leq \alpha$. 
Hence, 
\begin{align*}
\ex{}{F_{\beta,S}(w_T) - F_{\beta,S}(w_{\beta,S}^*)} &\leq O\left(\frac{\beta R^2 D^2}{T}\right) + \log(T)\big(D\alpha + s\log(|\cV|)\big) \\
& = O\left(\frac{\beta R^2 D^2}{T}\right) + \log(T)\left( \frac{L_0 R D}{n\log(n)}+ \frac{L_0 R D \sqrt{8T\log(1/\delta)}\log(|\cV|)}{{n\varepsilon}}\right),
\end{align*}
where the second equality follows from the setting of $\alpha=\frac{L_0 R}{n\log(n)}$ and the noise parameter $s$.

Using the same argument as in the proof of Lemma \ref{thm:l2_glm}, we arrive at the following bound on the excess empirical risk for the unsmoothed empirical loss $F_{S}$:
\begin{align*}
\ex{}{F_{S}(w_T) - F_{S}(w_S^*)}
& = O\left( \frac{\beta R^2 D^2}{T} + \frac{L_0 R D \sqrt{72T\log(1/\delta)}\log(|\cV|)\log(T)}{{n\varepsilon}} + \frac{L_0 R D\log(T)}{n\log(n)} + \frac{L_0^2}{\beta}\right).
\end{align*}

By the setting of $\beta=\frac{L_0 \sqrt{n\varepsilon}}{RD\log^{1/4}(1/\delta)\sqrt{\log(|\cV|)\log(n)}}$ and $T=\frac{n\varepsilon}{\log(|\cV|)\log(n)\sqrt{\log(1/\delta)}}$,
\[
\ex{}{F_{S}(w_T) - F_{S}(w_S^*)}
 = O\left( \frac{L_0 R D \log^{1/4}(1/\delta)\sqrt{\log(|\cV|)\log(n)}}{\sqrt{n\varepsilon}}\right).
\]

Via a standard Rademacher-complexity argument, we know that the generalization error of GLLs is bounded as $O\left(\frac{L_0 R D \sqrt{\log{d}}}{\sqrt{n}}\right)$ (see \cite{shalev2014understanding} Theorem 26.15). This gives the claimed bound. 

The privacy guarantee follows almost the same argument as in \cite{TTZ15a}. 
Note that the sensitivity of the approximate gradients generated by $\oracle$ is at most $\frac{3 L_0 R}{n}$ since $f_{\beta}$ is $(2 L_0 R)$-Lipschitz and the error due to the approximate oracle is less than $L_0R$. 
We then guarantee privacy via a straightforward application of the Report-Noisy-Max algorithm \cite{DR14,bhaskar2010discovering} and advanced composition for differential privacy. 

\end{proof}

\section{Algorithms for Non-convex Smooth Losses}\label{sec:smooth-ncnvx}
In this section, we describe differentially private algorithms for non-convex smooth stochastic optimization in the $\ell_p$-setting for $1\leq p\leq 2$. We provide formal convergence guarantees in terms of the stationarity gap (see (\ref{eqn:st-gap}) in Section~\ref{sec:Preliminaries}). Our algorithms are inspired by the variance-reduced stochastic Frank-Wolfe algorithm \cite{zhang2020one}. 
However, our algorithms involve several crucial differences from their non-private counterpart. In particular, they are divided into a number of rounds $R=O(\log(n))$, where each round $r\in \{0, \ldots, R-1\}$ involves $2^r$ updates for the iterate. Each round $r$ starts by computing a fresh estimate for the gradient of the population risk at the current iterate based on a large batch of data points, then such gradient estimate is updated recursively using disjoint batches of decreasing size sampled across the $2^r$ iterations of that round. Using this round-based structure and batch schedule, together with carefully tuned step sizes, allows us to effectively control the privacy budget while attaining small stationarity gap w.r.t.~the population risk. Moreover, our algorithms make a \emph{single pass} on the input sample, i.e., they run in linear time. 


In this section, we assume that $\forall z\in\cZ, ~f(\cdot, z)$ is $L_0$-Lipschitz and $L_1$-smooth loss in the respective $\ell_p$ norm. \rnote{I moved this paragraph from the appendix. The appendix now does not have anything related to this section.} Our algorithms can be applied to general spaces whose dual has a sufficiently smooth norm. To quantify this property, we use the notion of {\em regular spaces} \cite{Juditsky:2008}. Given $\kappa\geq 1$, we say a normed space $(\bE,\norm{\cdot})$ is $\kappa$-regular, if there exists $1\leq \kappa_+\leq\kappa$ and a norm $\|\cdot\|_+$ such that $(\bE,\|\cdot\|_{+})$ is $\kappa_+$-smooth, i.e.,
\begin{equation}\label{eqn:k_smooth}
    \|x+y\|_+^2 \leq \|x\|_+^2 +\langle \nabla(\|\cdot\|_+^2)(x),y\rangle+\kappa_+ \|y\|_+^2 \qquad(\forall x,y\in \bE),
\end{equation}
and $\norm{\cdot}$ and $\|\cdot\|_{+}$ are equivalent with constant $\sqrt{\kappa/\kappa_+}$:
\begin{equation}\label{eqn:k_equiv_norm}
 \norm{x}^2\leq \|x\|_+^2\leq \frac{\kappa}{\kappa_+} \norm{x}^2 \qquad(\forall x\in\bE).    
\end{equation}
\noindent One relevant fact is that $d$-dimensional $\ell_q$ spaces, $2\leq q\leq \infty$, are $\kappa$-regular with $\kappa=\min\left(q-1, 2\log d\right)$.
Also, if $\|\cdot\|$ is a polyhedral norm defined over a space $\bE$ with unit ball ${\cal B}_{\|\cdot\|}=\mbox{conv}({\cal V})$, then 
its dual $(\bE, \dual{\cdot})$ is $(2\log |{\cal V}|)$-regular.

\subsection{Algorithm for Polyhedral and $\ell_1$ Settings}\label{sec:polySFW}
We consider the {\em polyhedral} setup, namely, we consider a normed space $(\bE,\|\cdot\|)$, where the unit ball w.r.t. the norm, ${\cal B}_{\|\cdot\|}$ is a convex polytope with at most $J$ vertices. The feasible set $\cW$, is a polytope with at most $J$ vertices and $\|\cdot\|$-diameter $D>0$. 
\begin{algorithm}[!h]
	\caption{$\cA_\polyfw$: Private Polyhedral Stochastic Frank-Wolfe Algorithm}
	\begin{algorithmic}[1]
		\REQUIRE Dataset $S =  (z_1,\ldots z_n) \in \cZ^n$, 
		~privacy parameters $(\varepsilon,\delta)$, ~ polyhedral set $\cW$ with $J$ vertices $\cV = (v_1,\ldots,v_J)$, number of rounds $R$, batch size $b$, step sizes $\left(\eta_{r, t}: r=0,\ldots, R-1, ~t=0, \ldots, 2^{r}-1\right)$.
		
	
	    \STATE Choose an arbitrary initial point $w_0^0 \in \cW$ 
	    \FOR{$r =0$ to $R-1$}
	    \STATE Let $s_r= 2D(L_0+L_1D)\frac{2^r\sqrt{\log(1/\delta)}}{b\varepsilon}$
		\STATE Draw a batch $B_r^{0}$ of $b$ samples without replacement from $S$ 
		\STATE Compute ${\grd_r^0} = \frac{1}{b} \sum_{z\in B_r^0} \nabla f(w_r^0,z)$ \label{stp:grd0}
		
		\STATE $v_r^0 = \argmin\limits_{v \in \cV}{\{\ip{v}{\grd_r^0} + u_r^0(v)\}}$, where $u_r^0(v) \sim \lap\left(s_r\right)$\label{stp:v_r_0}
		
		\STATE $w^1_r \leftarrow (1-\eta_{r, 0})w^0_r + \eta_{r,0} v_r^0$

		\FOR{$t =1$ to $2^r-1$}
		
		\STATE Draw a batch $B_r^t$ of $b/(t+1)$ samples without replacement from $S$.    
		
		\STATE Compute $\Delta_r^t = \frac{t+1}{b}\sum_{z\in B_r^t}\left(\nabla f(w_r^t, z) - \nabla f(w_r^{t-1}, z) \right)$ \label{stp:grad-var}
		\STATE $\grd_r^t = (1-\eta_{r, t})\left(\grd_r^{t-1} + \Delta_r^t \right) + \eta_{r, t} \frac{t+1}{b}\sum_{z\in B_r^t}\nabla f(w_r^t,z)$\label{stp:grad_est}
		
		\STATE Compute $v_r^t = \argmin_{v \in \cV}{\ip{v}{\grd_r^t}+u_r^t(v)}$, where $u_r^t(v) \sim \lap\left(s_r\right)$\label{stp:v_r_t}
		\STATE $w_r^{t+1} \leftarrow (1-\eta_{r,t})w_r^{t} + \eta_{r,t} v_r^t$ \label{stp:update}
		\ENDFOR
		\STATE $w_{r+1}^0=w_r^{2^r}$
		\ENDFOR
		\STATE Output $\widehat{w}$ uniformly chosen from $\big(w_r^t: r\in \{0,\ldots, R-1\}, t\in \{0, \ldots, 2^{r}-1\}\big)$
	\end{algorithmic}
	\label{alg:polySFW}
\end{algorithm}

\rnote{I moved the paragraph below from the appendix.}
\paragraph{Remark concerning the choice of parameters $R$ and $b$\,:} Note that the total number of samples used the algorithm is $\sum_{r=0}^{R-1}\sum_{t=0}^{2^r-1}b/(t+1)\leq b\sum_{r=0}^R\left(\ln(2^r)+1\right)=b\sum_{r=0}^R\left(r\ln(2)+1\right)<bR^2$. Moreover, the batch drawn in each iteration $(r, t)$ is $b/(t+1)$. Hence, for the algorithm to be properly defined, it suffices to have $bR^2\leq n$ and $b\geq 2^R$. Note that our choices of $R$ and $b$ (in Lemma~\ref{lem:grad-error-polySFW} and Theorem~\ref{thm:convergence-polySFW} below) satisfy these conditions. 
\noindent Note also that we assume w.l.o.g. that $n$ is large enough so that the claimed bound on the stationarity gap is non-trivial. Hence, the choice of $R$ is meaningful. 



The formal guarantees of Algorithm~\ref{alg:polySFW} are stated below. 
\mnote{removed reference to appendix}

\begin{thm}\label{thm:privacy-polySFW}
Let $\eta_{r, t}=\frac{1}{\sqrt{t+1}}~\forall r, t$. Then, Algorithm~\ref{alg:polySFW} is $(\varepsilon, \delta)$-differentially private.
\end{thm}
\begin{proof}
Since the batches used in different rounds $r=0, \ldots, R-1$ are disjoint, it suffices to prove the privacy guarantee for a given round $r.$ The rest of the proof follows by parallel composition of differential privacy and the fact that differential privacy is closed under post processing. For notational brevity, let $g_r^t=\frac{t+1}{b}\sum_{z\in B_r^t}\nabla f(w_r^t, z)$. By unravelling the recursion in the gradient estimator (Step~\ref{stp:grad_est} of Algorithm~\ref{alg:polySFW}) and using the setting of $\eta_{r, t}=\frac{1}{\sqrt{t+1}}$, we have for any $t\in [2^{r}-1]$: 
\begin{align}
    \grd_r^t=a^{(1)}_t\cdot\grd_r^0+\sum_{k=1}^{t}\left(a^{(k)}_t\cdot\Delta_r^k + c^{(k)}_t\cdot g_r^k\right) \label{eqn:unravel_recur_polySFW}
\end{align}
where, for all $k\in [t],$ $a^{(k)}_t=\prod_{j=k}^t (1-\frac{1}{\sqrt{j+1}})$ and $c^{(k)}_t=\frac{1}{\sqrt{k+1}}\prod_{j=k+1}^t (1-\frac{1}{\sqrt{j+1}}).$ Note also that $a^{(k)}_t <1$ and $c^{(k)}_t <1$ for all $t, k$. 

Let $S, S'$ be any neighboring datasets (i.e., differing in exactly one data point). Let $\grd_r^t,\left\{\Delta_r^k: k\in [t]\right\}, \left\{g_r^k: k\in [t]\right\}$ be the quantities above when the input dataset is $S$; and let $\grd_r^{'t}, \left\{\Delta_r^{'k}, : k\in [t]\right\}, \left\{g_r^{'k}:  k\in [t]\right\}$ be the corresponding quantities when the input dataset is $S'$. Now, since the batches $B_r^0, \ldots, B_r^{t}$ are disjoint, changing one data point in the input dataset can affect at most one term in the sum (\ref{eqn:unravel_recur_polySFW}) above, i.e., it affects either the $\grd_r^0$ term, or exactly one term corresponding to some $k\in [t]$ in the sum on the right-hand side. Moreover, since $f$ is $L_0$-Lipschitz, we have $\dual{\grd_r^0-\grd_r^{'0}}\leq L_0/b$, and $\dual{g_r^t-g_r^{'t}}\leq L_0(t+1)/b$. Also, by the  $L_1$-smoothness of $f$ and the form of the update rule (Step~\ref{stp:update} of Algorithm~\ref{alg:polySFW}), for any $k\in \{1, \ldots, 2^r-1\}$, we have $\dual{\nabla f(w_r^k,z)-\nabla f(w_r^{k-1},z)}\leq L_1\norm{w_r^k-w_r^{k-1}}\leq L_1D\eta_{r, k}\leq L_1D/\sqrt{k+1}$. 
Hence, $\dual{\Delta_r^k-\Delta_r^{' k}}\leq \frac{k+1}{b}\frac{L_1D}{\sqrt{k+1}}=L_1 D\sqrt{k+1}/b$. Using these facts, it is then easy to see that for any $t\in [2^r-1]$, 
$$\dual{\grd_r^t-\grd^{' t}_r}\leq \max\left(\frac{L_0}{b}, \frac{(L_0+L_1 D)\sqrt{t+1}}{b}\right)\leq \frac{(L_0+L_1 D)2^{r/2}}{b}.$$
Hence, for each $v\in\cV$, the global sensitivity of $\ip{v}{\grd_r^t}$ is upper bounded by $\frac{D(L_0+L_1 D)2^{r/2}}{b}$. 
\noindent By the privacy guarantee of the Report Noisy Max mechanism \cite{DR14,bhaskar2010discovering}, the setting of the Laplace noise parameter $s_r$ ensures that each iteration $t\in \{0, \ldots, 2^r-1\}$ is $\frac{\varepsilon 2^{-r/2}}{\sqrt{\log(1/\delta)}}$-DP. Thus, by advanced composition (Lemma~\ref{lem:adv_comp}) applied to the $2^r$ iterations in round $r$,  we conclude that the algorithm is $(\varepsilon, \delta)$-DP.
\end{proof}


\begin{thm}\label{thm:convergence-polySFW}
Let $R=\frac{2}{3}\log\left(\frac{n\varepsilon}{\log^2(J)\log^2(n)\sqrt{\log(1/\delta)}}\right)$, $b=\frac{n}{\log^2(n)}$, and $\eta_{r, t}=\frac{1}{\sqrt{t+1}}~~\forall r, t$. Let $\cD$ be any distribution over $\cZ$. Let $S\sim\cD^n$ be the input dataset. The output $\widehat{w}$ of Algorithm~\ref{alg:polySFW} satisfies
\[\E\left[\gap_{F_{\cD}}(\widehat{w})\right]=O\left(D(L_0+L_1D)\cdot\frac{\log^{2/3}(J)\log^{2/3}(n)\log^{1/6}(1/\delta)}{n^{1/3}\varepsilon^{1/3}}\right).\]
\end{thm}

The proof of convergence will rely on the following lemma.
\begin{lem}\label{lem:grad-error-polySFW}
Let $\cD$ be any distribution over $\cZ$. Let $S\sim\cD^n$ be the input dataset of Algorithm ~\ref{alg:polySFW}. Let the step sizes $\eta_{r, t}=\frac{1}{\sqrt{t+1}}~~\forall r, t$. For every $r\in \{0,\ldots, R-1\}, ~t \in \{0, \ldots, 2^r-1\}$, the recursive gradient estimator $\grd_r^t$ satisfies \vspace{-0.1cm} 
	 \[ \textstyle \ex{}{\dual{\grd^t_r-\nabla F_{\cD}(w_r^t)}}
	\leq  4 L_0 \sqrt{\frac{\log(J)}{b}} \left(1-\frac{1}{\sqrt{t+1}}\right)^{t+1} + 4 \left( L_1 D + L_0 \right)\frac{\log(J)}{\sqrt{b}}(t+1)^{1/4}.
	\]
\end{lem}

\begin{proof}
Recall that we consider the {\em polyhedral} setup, where the feasible set $\cW$ is a polytope with at most $J$ vertices. Since the norm is polyhedral, the dual norm is also polyhedral. Hence, $(\bE,\dual{\cdot})$ is $(2\log(J))$-regular as discussed earlier in this section.

Fix any $r\in \{0, \ldots, R-1\}$. For any $t\in \{1, \ldots, 2^r-1\},$ we can write
\begin{align*}
\grd_r^t - \nabla F_\cD(w_r^t) &=(1-\eta_{r, t})~\left[\grd_r^{t-1} - \nabla F_{\cD}(w_r^{t-1})\right] + (1-\eta_{r, t})~\left[\Delta_r^t - \left(\nabla F_{\cD}(w_r^t)- \nabla F_{\cD}(w_r^{t-1})\right) \right]\\
&\quad   + \eta_{r, t}~\left[\frac{t+1}{b}\sum_{z\in B_r^t}\nabla f(w_r^t,z) - \nabla F_{\cD}(w_r^t)   \right].
\end{align*}
Let $\bDel_r^t \triangleq \nabla F_{\cD}(w_r^t)- \nabla F_{\cD}(w_r^{t-1})$. 
Recall that $\dual{\cdot}$ is $(2\log(J))$-regular, 
and denote $\eqnorm{\cdot}$ the corresponding $\kappa_+$-smooth norm, where $1 \leq \kappa_+ \leq 2 \log(J)$. First we will bound the variance in $\eqnorm{\cdot}$, and then we will derive the result using the equivalence property~\eqref{eqn:k_equiv_norm}.
Let $\cQ_r^t$ be the $\sigma$-algebra generated by the randomness in the data and the algorithm up until iteration $(r, t)$, i.e., the randomness in $\left\{\left(B_k^j, \left(u_k^j(v): v\in\cV\right)\right): 0\leq k\leq r, 0\leq j\leq t\right\}$. Define $\gamma_r^t\triangleq \E\left[\eqnorm{\grd_r^t-\nabla F_{\cD}(w_r^t)}^2~\vert~\cQ_r^{t-1}\right]$. By property~\eqref{eqn:k_smooth}, observe that 
\begin{align*}
\gamma_r^t &\leq (1-\eta_{r, t})^2 \gamma_r^{t-1}+ \kappa_+ \ex{}{\eqnorm{(1-\eta_{r, t})\left( \Delta_r^t - \bDel_r^t \right) + \eta_{r,t} \left(\frac{t+1}{b}\sum_{z\in B_r^t}\nabla f(w_r^t, z) - \nabla F_{\cD}(w_r^t)   \right)}^2\Bigg\vert\cQ_r^{t-1}} \\
&\leq  (1-\eta_{r, t})^2 \gamma_r^{t-1} +2\kappa_+(1-\eta_{r, t})^2\ex{}{\eqnorm{ \Delta_r^t - \bDel_r^t }^2\Bigg\vert\cQ_r^{t-1}}
+2\kappa_+\eta_{r, t}^2\ex{}{\eqnorm{\frac{t+1}{b}\sum_{z\in B_r^t}\nabla f(w_r^t, z) - \nabla F_{\cD}(w_r^t) }^2\Bigg\vert\cQ_r^{t-1}}. 
\end{align*}
In the first inequality, we used the fact that $\ex{z\sim \cD}{\nabla f(w,z)} = \nabla F_\cD(w)$, $\ex{z\sim \cD}{\Delta_r^t} = \bDel_r^t$, and the independence of $\left( \grd_r^{t-1} - \nabla F_{\cD}(w_r^{t-1}) \right)$ and $(1-\eta_{r, t})\left( \Delta_r^t - \bDel_r^t \right) + \eta_{r,t} \left(\nabla f(w_r^t,z) - \nabla F_{\cD}(w_r^t)   \right)$ conditioned on $\cQ_r^{t-1}$. The second inequality follows by triangle inequality and the fact that $(a+b)^2 \leq 2a^2 + 2b^2$ for $a,b \in \re$. Hence, using~\eqref{eqn:k_equiv_norm}
and $L_1$-smoothness of the loss, we can obtain the following bound inductively:
\begin{align*}
\ex{}{\eqnorm{ \Delta_r^t - \bDel_r^t }^2\Bigg\vert\cQ_r^{t-1}}&=\ex{}{\eqnorm{ \frac{t+1}{b}\sum_{z\in B_r^t}\left(\nabla f(w_r^t, z) -\nabla f(w_r^{t-1}, z) -\bDel_r^t \right)}^2\Bigg\vert\cQ_r^{t-1}}\\ 
&\leq \frac{(t+1)^2}{b^2}\ex{}{\eqnorm{\sum_{z\in B_r^t\setminus\{z'\}}\left(\nabla f(w_r^t, z) -\nabla f(w_r^{t-1}, z) -\bDel_r^t \right)}^2\Bigg\vert\cQ_r^{t-1}}\\
&~ + \kappa_+\frac{(t+1)^2}{b^2}\ex{}{\eqnorm{\nabla f(w_r^t, z') -\nabla f(w_r^{t-1}, z') -\bDel_r^t }^2\Bigg\vert\cQ_r^{t-1}}\\
&\leq \kappa_+\frac{(t+1)^2}{b^2}\sum_{z\in B_r^t}\ex{}{\eqnorm{\nabla f(w_r^t, z) -\nabla f(w_r^{t-1}, z) -\bDel_r^t }^2\Bigg\vert\cQ_r^{t-1}}\\
&\leq \kappa\frac{(t+1)^2}{b^2}\sum_{z\in B_r^t}\ex{}{\dual{\nabla f(w_r^t, z) -\nabla f(w_r^{t-1}, z) -\bDel_r^t }^2\Bigg\vert\cQ_r^{t-1}}\\
&\leq \frac{4\left(L_1D\right)^2\log(J)\eta_{r, t}^2 \,(t+1)}{b}, 
\end{align*}
where the inequality before the last one follows from the fact that $\kappa_+\leq \kappa$, and the last inequality follows from the fact that $\kappa=2\log(J)$.
Similarly, since the loss is $L_0$-Lipschitz, using the same inductive approach, we can bound
\begin{align*}
& \ex{}{\eqnorm{\frac{t+1}{b}\sum_{z\in B_r^t}\nabla f(w_r^t,z) - \nabla F_{\cD}(w_r^t) }^2\Bigg\vert\cQ_r^{t-1}} \leq 
\frac{4L_0^2\log(J)\,(t+1)}{b}. 
\end{align*}
Using the above bounds and the setting of $\eta_{r, t}$, we reach the following recursion
\begin{align*}
    \gamma_r^t &\leq \left(1-\frac{1}{\sqrt{t+1}}\right)^2\gamma_r^{t-1}+\frac{8\kappa_+(L_0^2+L_1^2D^2)\log(J)}{b}.
\end{align*}
Unravelling the recursion, we can further bound $\gamma_r^{t}$ as:
\begin{align}
\gamma_r^t&\leq \gamma_r^0\left(1-\frac{1}{\sqrt{t+1}}\right)^{2t}+\frac{8\kappa_+(L_0^2+L_1^2D^2)\log(J)}{b}\sum_{j=0}^{t-1}\left(1-\frac{1}{\sqrt{t+1}}\right)^{2j}\nonumber\\
&\leq \gamma_r^0\left(1-\frac{1}{\sqrt{t+1}}\right)^{2t}+\frac{8\kappa_+(L_0^2+L_1^2D^2)\log(J)\sqrt{t+1}}{b},\label{eqn:recur_gamma_polySFW}
\end{align}
where the last inequality follows from the fact that $\sum_{j=0}^{t-1}(1-\frac{1}{\sqrt{t+1}})^{2j}\leq \frac{1}{1-(1-\frac{1}{\sqrt{t+1}})^2}\leq \sqrt{t+1}.$ 

Moreover, observe that we can bound $\gamma_r^0$ using the same inductive approach we used earlier: 
\begin{align*}
    \gamma_r^0&=\ex{}{\eqnorm{\frac{1}{b}\sum_{z\in B_r^0}\nabla f(w_r^0,z) - \nabla F_{\cD}(w_r^0) }^2\Bigg\vert\cQ_{r-1}^{2^{r-1}-1}}\\
    &\leq \frac{1}{b^2} \left( \ex{}{\eqnorm{\sum_{z\in B_r^0\setminus\{z'\}}\left(\nabla f(w_r^0, z)-\nabla F_\cD(w_r^0) \right)}^2\Bigg\vert\cQ_{r-1}^{2^{r-1}-1}}  + \kappa_+ \ex{}{\eqnorm{\nabla f(w_r^0, z')-\nabla F_\cD(w_r^0)}^2\Bigg\vert\cQ_{r-1}^{2^{r-1}-1}} \right) \\
    &\leq \frac{\kappa_+}{b^2}  \sum_{z\in B_r^0} \ex{}{\eqnorm{\nabla f(w_r^0, z)-\nabla F_\cD(w_r^0)}^2\Bigg\vert\cQ_{r-1}^{2^{r-1}-1}}\\
    &\leq \frac{4L_0^2\log(J)}{b}.
\end{align*}
Plugging this in (\ref{eqn:recur_gamma_polySFW}), we can finally arrive at 
\begin{align*}
\ex{}{\eqnorm{\grd_r^t-\nabla F_{\cD}(w_r^t)}^2}&\leq \frac{4L_0^2\log(J)}{b}\Big(1-\frac{1}{\sqrt{t+1}}\Big)^{2t}+\frac{8\kappa_+(L_0^2+L_1^2D^2)\log(J)\sqrt{t+1}}{b}\\
&\leq \frac{4L_0^2\log(J)}{b}\Big(1-\frac{1}{\sqrt{t+1}}\Big)^{2t}+\frac{16(L_0^2+L_1^2D^2)\log^2(J)\sqrt{t+1}}{b},
\end{align*}
where the last inequality follows from the fact that $\kappa_+\leq\kappa =2\log(J)$. 

By property~\eqref{eqn:k_equiv_norm} of regular norms and using Jensen's inequality together with the subadditivity of the square root, we reach the desired bound: 
\begin{align*}
\ex{}{\dual{\grd_r^t-\nabla F_{\cD}(w_r^t)}}&\leq\sqrt{\ex{}{\eqnorm{\grd_r^t-\nabla F_{\cD}(w_r^t)}^2}}\\
&\leq 4 L_0 \sqrt{\frac{\log(J)}{b}} \left(1-\frac{1}{\sqrt{t+1}}\right)^{t} + 4 \left( L_1 D + L_0 \right)\frac{\log(J)}{\sqrt{b}}(t+1)^{1/4}.
\end{align*}
\end{proof}

\paragraph{Proof of Theorem \ref{thm:convergence-polySFW}}
For any $r \in \{0, \ldots, R-1\}$ and $t\in \{0, \ldots, 2^r-1\}$, let $\alpha_r^t \triangleq \ip{v_r^t}{\grd_r^{t}} - \min_{v \in \cV}{\ip{v}{\grd_r^{t}}}$; and let $v^{\ast}_{r, t}=\argmin\limits_{v\in\cW}\ip{\nabla F_{\cD}(w_r^t)}{v-w_r^t}$.
By smoothness and convexity of $F_\cD$, observe
	\begin{align*}
	\textstyle     F_\cD(w_r^{t+1}) & \textstyle \leq F_\cD(w_r^{t}) + \ip{\nabla F_\cD(w_r^t)}{w_r^{t+1} - w_r^t} + \frac{L_1}{2}\|w_r^{t+1} - w_r^t\|^2\\
	&\textstyle \leq  F_\cD(w_r^{t}) +  \eta_{r,t} \ip{\nabla F_\cD(w_r^t) - \grd_r^t}{v_r^t - w_r^t} + \eta_{r, t} \ip{\grd_r^t}{v_r^t -w_r^t}+ \frac{L_1 D^2 \eta_{r, t}^2}{2}\\
	& \textstyle\leq F_\cD(w_r^{t}) +  \eta_{r, t} \ip{\nabla F_\cD(w_r^t) - \grd_r^t}{v_r^t - w_r^t}  + \eta_{r, t} \ip{\grd_r^t}{v^{\ast}_{r,t} -w_r^t} + \eta_{r, t} \alpha_r^t + \frac{L_1 D^2 \eta_{r, t}^2}{2} \\
	& \textstyle= F_\cD(w_r^{t}) +  \eta_{r, t} \ip{\nabla F_\cD(w_r^t) - \grd_r^t}{v_r^t - v^{\ast}_{r,t}} -\eta_{r, t}\ip{\nabla F_{\cD}(w_r^t)}{v^{\ast}_{r,t} -w_r^t} + \eta_{r,t} \alpha_r^t +\frac{L_1 D^2 \eta_{r,t}^2}{2} \\
	& \textstyle\leq F_\cD(w_r^{t}) +  \eta_{r, t} D \dual{\nabla F_\cD(w_r^t) - \grd_r^t} -\eta_{r, t}\gap_{F_{\cD}}(w_r^t) + \eta_{r,t} \alpha_r^t +\frac{L_1 D^2 \eta_{r,t}^2}{2}.
	\end{align*}
\cnote{I believe I have asked this in the past. But could it be possible to refine the upper bound above of $\eta_{r, t} D \dual{\nabla F_\cD(w_r^t)- \grd_r^t}$, in expectation? Something along the lines of upper bounding $\eta_{r, t} \ip{\nabla F_\cD(w_r^t) - \grd_r^t}{v_r^t - v^{\ast}_{r,t}}$, by estimating the proba of sampling a vertex with far different value than $v_{r,t}^{\ast}$. This is more of an open-ended question.}\rnote{That's a fair question. Let's think about it. (But, as far as Neurips submission is concerned, I believe we should keep what we have now for consistency.)} \cnote{Absolutely! I am just keeping this question for the record.}
Hence, we have  
$$\E[\gap_{F_{\cD}}(w_r^t)]\leq \frac{\E[F_{\cD}(w_r^t)-F_{\cD}(w_r^{t+1})]}{\eta_{r,t}}+\frac{L_1 D^2 \eta_{r,t}}{2}+D\,\E\left[\dual{\grd_r^t-\nabla F_{\cD}(w_r^t)}\right]+\E[\alpha_r^t].$$
Note that by a standard argument $\ex{}{\alpha_r^t}\leq 2s_r\log(J)=\frac{4D(L_0+L_1D)2^r\log(J)\sqrt{\log(1/\delta)}}{b\varepsilon}$. Thus, given the bound on $\ex{}{\dual{\grd_r^t-\nabla F_{\cD}(w_r^t)}}$ from Lemma~\ref{lem:grad-error-polySFW}, we have
\begin{align*}
    \E[\gap_{F_{\cD}}(w_r^t)]\leq& \sqrt{t+1}\left(\E[F_{\cD}(w_r^t)-F_{\cD}(w_r^{t+1})]\right)+\frac{L_1 D^2}{2\sqrt{t+1}}+4L_0D\sqrt{\frac{\log(J)}{b}} \left(1-\frac{1}{\sqrt{t+1}}\right)^{t}\\
    &+ 4D \left( L_1 D + L_0 \right)\frac{\log(J)}{\sqrt{b}}(t+1)^{1/4}+4D(L_0+L_1D)\frac{\log(J)\sqrt{\log(1/\delta)}}{b\varepsilon}\,2^r.
\end{align*}
For any given $r \in \{0, \ldots, R-1\},$ we now sum both sides of the above inequality over $t\in \{0, \ldots, 2^{r}-1\}$. 

Let $\Gamma_r\triangleq\sum_{t=0}^{2^{r}-1}\sqrt{t+1}\left(\E[F_{\cD}(w_r^t)-F_{\cD}(w_r^{t+1})]\right).$ Observe that
\begin{align*}
 \sum_{t=0}^{2^r-1}\E[\gap_{F_{\cD}}(w_r^t)]\leq&~ \Gamma_r+\frac{L_1D^2}{2}\sum_{t=1}^{2^r}\frac{1}{\sqrt{t}}+4L_0D\sqrt{\frac{\log(J)}{b}}\,\sum_{t=0}^{2^r-1}\Big(1-\frac{1}{\sqrt{t+1}}\Big)^t\\
 &\quad +4D(L_0+DL_1)\frac{\log(J)}{\sqrt{b}}\sum_{t=1}^{2^r}t^{1/4}+4D(L_0+L_1D)\frac{\log(J)\sqrt{\log(1/\delta)}}{b\varepsilon}2^{2r}\\
 \leq &~\Gamma_r + L_1D^2\, 2^{r/2}+4L_0D\sqrt{\frac{\log(J)}{b}}\sum_{t=0}^{2^r-1}(1-2^{-r/2})^t\\
 &\quad + 8D(L_0+DL_1)\frac{\log(J)}{\sqrt{b}}2^{5r/4} +4D(L_0+L_1D)\frac{\log(J)\sqrt{\log(1/\delta)}}{b\varepsilon}2^{2r}\\
 \leq &~\Gamma_r + L_1D^2 2^{r/2}+4L_0D\sqrt{\frac{\log(J)}{b}}2^{r/2}+8D(L_0+L_1D)\frac{\log(J)}{\sqrt{b}}2^{5r/4}\\ 
 &\quad +4D(L_0+L_1D)\frac{\log(J)\sqrt{\log(1/\delta)}}{b\varepsilon}2^{2r}.
\end{align*}
Next, we bound $\Gamma_r$. Before we do so, note that for all $z\in\cZ$, $f(\cdot, z)$ is $L_0$-Lipschitz and the $\norm{\cdot}$-diameter of $\cW$ is bounded by $D$, hence, w.l.o.g., we will assume that the range of $f(\cdot, z)$ lies in $[-L_0 D, L_0 D]$. This implies that the range of $F_{\cD}$ lies in $[-L_0D, L_0D]$. Now, observe that
\begin{align*}
    \Gamma_r=&\sum_{t=0}^{2^{r}-1}\sqrt{t+1}\,\,\left(\E[F_{\cD}(w_r^t)-F_{\cD}(w_r^{t+1})]\right)\\
    =&\sum_{t=0}^{2^{r}-1}\left(\sqrt{t+1}\,\,\ex{}{F_\cD(w_r^t)}-\sqrt{t+2}\,\,\ex{}{F_{\cD}(w_r^{t+1})}\right)+\sum_{t=0}^{2^{r}-1}\left(\sqrt{t+2} -\sqrt{t+1}\right)\,\ex{}{F_\cD(w_r^{t+1})}\\
    \leq& \sum_{t=0}^{2^{r}-1}\left(\sqrt{t+1}\,\,\ex{}{F_\cD(w_r^t)}-\sqrt{t+2}\,\,\ex{}{F_{\cD}(w_r^{t+1})}\right)+L_0 D\sum_{t=0}^{2^{r}-1}\left(\sqrt{t+2} -\sqrt{t+1}\right)
\end{align*}
Note that both sums on the right-hand side are telescopic. Hence, we get
\begin{align*}
    \Gamma_r \leq& \ex{}{F_\cD(w_r^0)-\sqrt{2^{r}+1}F_\cD(w_r^{2^r})}+L_0D \,2^{r/2}\\
    =& \ex{}{F_\cD(w_r^0)-F_\cD(w_r^{2^r})}-\left(\sqrt{2^{r}+1}-1\right)\ex{}{F_\cD(w_r^{2^r})}+ L_0D\, 2^{r/2}\\
    \leq& 3 L_0D \,2^{r/2}.
\end{align*}
Thus, we arrive at 
\begin{align*}
 \sum_{t=0}^{2^r-1}\E[\gap_{F_{\cD}}(w_r^t)]\leq&~3D(L_0 +L_1D) 2^{r/2}+4L_0D\sqrt{\frac{\log(J)}{b}}2^{r/2}+8D(L_0+L_1D)\frac{\log(J)}{\sqrt{b}}2^{5r/4}\\ 
 &\quad +4D(L_0+L_1D)\frac{\log(J)\sqrt{\log(1/\delta)}}{b\varepsilon}2^{2r}.
\end{align*}
Now, summing over all rounds $r\in \{0, \ldots, R-1\}$, we have
\begin{align*}
 \sum_{r=0}^{R-1}\,\,\sum_{t=0}^{2^r-1}\E[\gap_{F_{\cD}}(w_r^t)]\leq&~9D(L_0+L_1D)2^{R/2}+12L_0D\sqrt{\frac{\log(J)}{b}}2^{R/2}+6D(L_0+L_1D)\frac{\log(J)}{\sqrt{b}}2^{5R/4}\\ &~ +2D(L_0+L_1D)\frac{\log(J)\sqrt{\log(1/\delta)}}{b\varepsilon}2^{2R}.
\end{align*}

Recall that the output $\widehat{w}$ is uniformly chosen from the set of all $2^R$ iterates. By taking expectation with respect to that random choice and using the above, we get
\begin{align*}
\E[\gap_{F_{\cD}}(\widehat{w})]&=\frac{1}{2^R}\sum_{r=0}^{R-1}\,\,\sum_{t=0}^{2^r-1}\E[\gap_{F_{\cD}}(w_r^t)]\\
&\leq 9D(L_0+L_1D)2^{-R/2}+12L_0D\sqrt{\frac{\log(J)}{b}}2^{-R/2}+6D(L_0+L_1D)\frac{\log(J)}{\sqrt{b}}2^{R/4}\\ &~ +2D(L_0+L_1D)\frac{\log(J)\sqrt{\log(1/\delta)}}{b\varepsilon}2^{R}.
\end{align*}
Recall that $R=\frac{2}{3}\log\left(\frac{n\varepsilon}{\log^2(J)\log^2(n)\sqrt{\log(1/\delta)}}\right)$ and $b=\frac{n}{\log^2(n)}$. Hence, we have
\begin{align*}
\E[\gap_{F_{\cD}}(\widehat{w})]&\leq 9D(L_0+L_1D)\left(\frac{\log^2(J)\sqrt{\log(1/\delta)}\log^2(n)}{n\varepsilon}\right)^{1/3}+12L_0D\sqrt{\frac{\log(J)\log^2(n)}{n}}\left(\frac{\log^2(J)\sqrt{\log(1/\delta)}\log^2(n)}{n\varepsilon}\right)^{1/3}\\
& +6D(L_0+L_1D)\frac{\varepsilon^{1/6}}{\log^{1/3}(n)\log^{1/12}(1/\delta)}\left(\frac{\log^2(J)}{n}\right)^{1/3}+2D(L_0+L_1D)\left(\frac{\log^2(J)\sqrt{\log(1/\delta)}\log^2(n)}{n\varepsilon}\right)^{1/3}\\
&=~ O\left(D(L_0+L_1D)\left(\frac{\log^2(J)\log^2(n)\sqrt{\log(1/\delta)}}{n\varepsilon}\right)^{1/3}\right),
\end{align*}
which is the claimed bound.

\subsection{Algorithm for $\ell_p$ Settings when $1<p\leq 2$}\label{sec:noisySFW}

\begin{algorithm}[!h]
	\caption{$\cA_\nsfw$: Private Noisy Stochastic Frank-Wolfe Algorithm for $\ell_p$ DP-SO, $1<p\leq 2$}
	\begin{algorithmic}[1]
		\REQUIRE Private dataset $S =  (z_1,\ldots z_n) \in \cZ^n$, 
		~privacy parameters $(\varepsilon,\delta)$, a number $p\in (1, 2]$ ~ feasible set $\cW\subset \re^d$ with $\norm{\cdot}_p$-diameter $D$, number of rounds $R$, batch size $b$, step sizes $\left(\eta_{r, t}: r=0,\ldots, R-1, ~t=0, \ldots, 2^{r}-1\right)$
	
	    \STATE Choose an arbitrary initial point $w_0^0 \in \cW$
	    \FOR{$r =0$ to $R-1$}
	    \STATE Let $\sigma^2_{r, 0}=\frac{16 L_0^2d^{2/p-1}\log(1/\delta)}{b^2\varepsilon^2}$
		\STATE Draw a batch $B_r^{0}$ of $b$ samples without replacement from $S$
		\STATE Compute $\tgrd_r^0 = \frac{1}{b} \sum_{z\in B_r^0} \nabla f(w_r^0,z) + N_r^0,~ ~N_r^0\sim \cN\left(0, \sigma_{r, 0}^2\mathbb{I}_d\right)$ \label{stp:tgrd0}
		
		\STATE $v_r^0 = \argmin\limits_{v \in \cW}\ip{v}{\tgrd_r^0}$ \label{stp:tv_r_0}
		
		\STATE $w^1_r \leftarrow (1-\eta_{r, 0})w^0_r + \eta_{r,0} v_r^0$

		\FOR{$t =1$ to $2^r-1$}
		
		 \STATE Let $\sigma^2_{r, t}=\frac{16 L_0^2(t+1)^2d^{2/p-1}\log(1/\delta)}{b^2\varepsilon^2}, ~\tsigma^2_{r, t}=\frac{16 L_1^2D^2\eta_{r,t}^2(t+1)^2d^{2/p-1}\log(1/\delta)}{b^2\varepsilon^2}$
		\STATE Draw a batch $B_r^t$ of $b/(t+1)$ samples without replacement from $S$    
		
		\STATE Let $\Delta_r^t= \frac{t+1}{b}\sum_{z\in B_r^t}\left(\nabla f(w_r^t, z) - \nabla f(w_r^{t-1}, z) \right),$ and let $g_r^t=\frac{t+1}{b}\sum_{z\in B_r^t}\nabla f(w_r^t, z)$
		\STATE Compute $\tdel_r^t= \Delta_r^t+ \tN_r^t,~~ \tN_r^t\sim \cN\left(0, \tsigma_{r, t}^2\mathbb{I}_d\right)$  \label{stp:tgrad-var}
		\STATE  Compute $\tg_r^t=g_r^t+N_r^t,~~N_r^t\sim \cN\left(0, \sigma_{r, t}^2\mathbb{I}_d\right)$\label{stp:tg}
		\STATE $\tgrd_r^t = (1-\eta_{r, t})\left(\tgrd_r^{t-1} + \tdel_r^t \right) + \eta_{r, t} \tg_r^t$\label{stp:tgrad_est}
		
		\STATE Compute $v_r^t = \argmin_{v \in \cW}\ip{v}{\tgrd_r^t}$ \label{stp:tv_r_t}
		\STATE $w_r^{t+1} \leftarrow (1-\eta_{r,t})w_r^{t} + \eta_{r,t} v_r^t$ \label{stp:t_update}
		\ENDFOR
		\STATE $w_{r+1}^0=w_r^{2^r}$
		\ENDFOR
		\STATE Output $\widehat{w}$ uniformly chosen from the set of all iterates $\left(w_r^t: r=0, \ldots, R-1, t=0, \ldots, 2^r-1\right)$ 
	\end{algorithmic}
	\label{alg:nSFW}
\end{algorithm}

Our algorithm in this setting (Algorithm \ref{alg:nSFW}) \mnote{Added algorithm reference} has a similar structure to Algorithm~\ref{alg:polySFW} in Section~\ref{sec:polySFW}, except for the following few, but crucial, differences. First, for all iterations $(r, t)$: the recursive gradient estimate $\grd_r^t$ and the gradient variation estimate $\Delta_r^t$ are replaced with noisy versions $\tgrd_r^t$ and $\tdel_r^t$ obtained by adding Gaussian noise to the respective quantities. 
\rnote{I removed a few lines in the earlier description since they are redundant now given that we include the formal alg. description.}
\noindent The second difference here pertains to the way the iterates are updated, which now becomes $w_r^{t+1}=(1-\eta_{r,t})w_r^t+\eta_{r,t}\argmin\limits_{v\in \cW}\ip{v}{\tgrd_r^t}$. 
\noindent Finally, we use a different setting for the number of rounds $R$ than the one used earlier. Below, we state the formal guarantees of this algorithm, which we refer to as \emph{noisy stochastic Frank-Wolfe}, $\cA_\nsfw$. 
\mnote{removed appendix reference} 

\begin{thm}\label{thm:privacy-noisySFW}
Algorithm $\cA_\nsfw$ 
\mnote{Removed appendix reference}
is $(\varepsilon, \delta)$-DP.
\end{thm}
\begin{proof}
Note that it suffices to show that for any given $(r, t),$ $r\in \{0, \ldots, R-1\},~t\in [2^r-1]$, computing $\tgrd_r^0$ (Step~\ref{stp:tgrd0} in Algorithm~\ref{alg:nSFW}) satisfies $(\varepsilon, \delta)$-DP, and computing $\tdel_r^t, \tg_r^t$ (Steps~\ref{stp:tgrad-var} and \ref{stp:tg}) satisfies $(\varepsilon, \delta)$-DP. Assuming we can show that this is the case, then note that at any given iteration $(r, t)$, the gradient estimate $\tgrd_r^{t-1}$ from the previous iteration is already computed privately. Since differential privacy is closed under post-processing, then the current iteration is also $(\varepsilon, \delta)$-DP. Since the batches used in different iterations are disjoint, then by parallel composition, the algorithm is $(\varepsilon, \delta)$-DP. Thus, it remains to show that for any given $(r, t)$, the steps mentioned above are computed in $(\varepsilon, \delta)$-DP manner. Let $S, S'$ be neighboring datasets (i.e., differing in exactly one point). Let $\tgrd_r^0, \tdel_r^t, \tg_r^t$ be the quantities above when the input dataset is $S$; and let $\tgrd_r^{' 0}, \tdel_r^{' t}, \tg_r^{' t}$ be the corresponding quantities when the input dataset is $S'$. Note that the $\ell_2$-sensitivity of $\tgrd_r^0$ can be bounded as $\twonorm{\tgrd_r^0-\tgrd^{'0}_r}\leq d^{\frac{1}{p}-\frac{1}{2}}\dual{\tgrd_r^0-\tgrd^{'0}_r}\leq \frac{L_0d^{\frac{1}{p}-\frac{1}{2}}}{b}$, where the dual norm here is $\dual{\cdot}=\norm{\cdot}_q$ where $q=\frac{p}{p-1}$. Similarly, we can bound the $\ell_2$-sensitivity of $\tg_r^t$ as $\twonorm{\tg_r^t-\tg^{'t}_r}\leq\frac{L_0d^{\frac{1}{p}-\frac{1}{2}}(t+1)}{b}$. Also, by the $L_1$-smoothness of the loss, we have  $\twonorm{\tdel_r^t-\tdel^{'t}_r}\leq d^{\frac{1}{p}-\frac{1}{2}}\dual{\tdel_r^t-\tdel^{'t}_r}\leq \frac{L_1 D \eta_{r,t} d^{\frac{1}{p}-\frac{1}{2}}(t+1)}{b}$. Given these bounds and the settings of the noise parameters in the algorithm, the argument follows directly by the privacy guarantee of the Gaussian mechanism.
\end{proof}

\rnote{@Mike: Lemma 15 (and its proof) needs to come before we prove the theorem below (like you did in the previous sub-section) because the proof of the theorem below depends on Lemma 15} \mnote{They are now swapped}

\begin{thm}\label{thm:convergence-noisySFW}
Consider the $\ell_p$ setting of non-convex smooth stochastic optimization, where $1<p\leq 2$. Let $\kappa=\min\left(\frac{1}{p-1}, 2\log(d)\right)$ and $\tkappa=1+\log(d)\cdot\ind(p<2)$. In $\cA_{\nsfw}$, let $R=\frac{4}{5}\log\left(\frac{n\varepsilon}{\sqrt{d\tkappa\log(1/\delta)}\,\kappa^{5/3}\log^2(n)}\right)$, $b=\frac{n}{\log^2(n)}$, and $\eta_{r, t}=\frac{1}{\sqrt{t+1}}~~\forall r, t$. Let $\cD$ be any distribution over $\cZ$, and $S\sim\cD^n$ be the input dataset. The output $\widehat{w}$ satisfies:
$$\ex{}{\gap_{F_{\cD}}(\widehat{w})}=O\left(D(L_0+L_1D)\kappa^{2/3}\left(\frac{\log^{2/3}(n)}{n^{1/3}}+\frac{d^{1/5}\,\tkappa^{1/5}\log^{1/5}(1/\delta)\log^{4/5}(n)}{n^{2/5}\varepsilon^{2/5}}\right)\right).$$
\end{thm}
\noindent \emph{Note that for the Euclidean setting, we have $\kappa=\tkappa=1$ in the above bound.}

As before, the first step of the proof is given by the following lemma, which gives a bound on the error in the gradient estimates in the dual norm. 
\begin{lem}\label{lem:grad-error-noisySFW}
Let $\cD$ be any distribution over $\cZ$, and $S\sim\cD^n$ be the input dataset. For the same settings of parameters in Theorem~\ref{thm:convergence-noisySFW}, the gradient estimate $\tgrd_r^t$ satisfies the following for all $r, t$: 
\begin{align*}
    \ex{}{\dual{\tgrd^t_r-\nabla F_{\cD}(w_r^t)}}
	\leq&  8 L_0 \left(\sqrt{\frac{\kappa}{b}}+\frac{\sqrt{d\kappa\tkappa\log(1/\delta)}}{b\varepsilon}\right) \left(1-\frac{1}{\sqrt{t+1}}\right)^{t+1} \\
	&+ 16 \left( L_1 D + L_0 \right)\left(\frac{\kappa}{\sqrt{b}}(t+1)^{1/4}+\frac{\sqrt{d\kappa\tkappa\log(1/\delta)}}{b\varepsilon}(t+1)^{3/4}\right).
\end{align*}	
\end{lem}
\begin{proof}
Note that for the $\ell_p$ space, where $p\in (1, 2]$, the dual  is the $\ell_q$ space where $q=\frac{p}{p-1}\geq 2$. To keep the notation consistent with the rest of the paper, in the sequel, we will be using $\dual{\cdot}$ to denote the dual norm $\norm{\cdot}_q$ unless specific reference to $q$ is needed. As discussed earlier in this section, 
the dual space $\ell_q$ is $\kappa$-regular with $\kappa=\min\left(q-1, 2\log(d)\right)=\min\left(\frac{1}{p-1}, 2\log(d)\right)$.

Fix any $r\in \{0, \ldots, R-1\}$ and $t\in \{1, \ldots, 2^{r}-1\}$. As we did in the proof of Lemma~\ref{lem:grad-error-polySFW}, we write 
\begin{align*}
\tgrd_r^t - \nabla F_\cD(w_r^t) &=(1-\eta_{r, t})~\left[\tgrd_r^{t-1} - \nabla F_{\cD}(w_r^{t-1})\right] + (1-\eta_{r, t})~\left[\tdel_r^t - \bDel_r^t \right]\\
&\quad   + \eta_{r, t}~\left[\tg_r^t - \nabla F_{\cD}(w_r^t)   \right].
\end{align*}
where $\bDel_r^t \triangleq \nabla F_{\cD}(w_r^t)- \nabla F_{\cD}(w_r^{t-1})$. 

Let $\eqnorm{\cdot}$ denote the $\kappa_+$-smooth norm associated with $\dual{\cdot}$ (as defined by the regularity property, in the beginning of this section). 
Note that by $\kappa$-regularity of $\dual{\cdot}$, such norm exists for some $1\leq\kappa_+\leq \kappa$. Let $\cQ_r^t$ be the $\sigma$-algebra induced by all the randomness up until the iteration indexed by $(r, t)$. Define $\gamma_r^t\triangleq \E\left[\eqnorm{\tgrd_r^t-\nabla F_{\cD}(w_r^t)}^2~\Big\vert~\cQ_r^{t-1}\right]$. Note by property~\eqref{eqn:k_smooth} of $\kappa$-regular norms, we have 
\begin{align}
\gamma_r^t &\leq (1-\eta_{r, t})^2 \gamma_r^{t-1}+ \kappa_+ \ex{}{\eqnorm{(1-\eta_{r, t})\left( \tdel_r^t - \bDel_r^t \right) + \eta_{r,t} \left(\tg_r^t - \nabla F_{\cD}(w_r^t)   \right)}^2\Bigg\vert\cQ_r^{t-1}} \nonumber\\
&\leq (1-\eta_{r, t})^2\, \gamma_r^{t-1}+ \kappa_+ \ex{}{\eqnorm{(1-\eta_{r, t})\left( \Delta_r^t - \bDel_r^t +\tN_r^t\right) + \eta_{r,t} \left(g_r^t - \nabla F_{\cD}(w_r^t) +N_r^t  \right)}^2\Bigg\vert\cQ_r^{t-1}} \nonumber\\
&\leq  (1-\eta_{r, t})^2 \gamma_r^{t-1} +2\kappa_+(1-\eta_{r, t})^2\, \ex{}{\eqnorm{ \Delta_r^t - \bDel_r^t +\tN_r^t}^2\Bigg\vert\cQ_r^{t-1}}
+2\kappa_+\eta_{r, t}^2\, \ex{}{\eqnorm{g_r^t- \nabla F_{\cD}(w_r^t)+N_r^t }^2\Bigg\vert\cQ_r^{t-1}} \nonumber\\
&\leq (1-\eta_{r, t})^2 \gamma_r^{t-1} +4\kappa_+(1-\eta_{r, t})^2\,\ex{}{\eqnorm{ \Delta_r^t - \bDel_r^t}^2\Bigg\vert\cQ_r^{t-1}}+4\kappa_+(1-\eta_{r, t})^2\,\ex{}{\eqnorm{ \tN_r^t}^2\Bigg\vert\cQ_r^{t-1}}\nonumber\\
&~~ +4\kappa_+\eta_{r, t}^2\,\ex{}{\eqnorm{g_r^t- \nabla F_{\cD}(w_r^t)}^2\Bigg\vert\cQ_r^{t-1}}+4\kappa_+\eta_{r, t}^2\,\ex{}{\eqnorm{N_r^t }^2\Bigg\vert\cQ_r^{t-1}}.\label{eqn:recur1}
\end{align}
where the last two inequalities follow from the triangle inequality. 

Now, using the same inductive approach we used in the proof of Lemma~\ref{lem:grad-error-polySFW}, we can bound
\begin{align*}
\ex{}{\eqnorm{ \Delta_r^t - \bDel_r^t }^2\Bigg\vert\cQ_r^{t-1}} &\leq \kappa\frac{(t+1)^2}{b^2}\sum_{z\in B_r^t}\ex{}{\dual{\nabla f(w_r^t, z) -\nabla f(w_r^{t-1}, z) -\bDel_r^t }^2\Bigg\vert\cQ_r^{t-1}} \leq \frac{2\kappa L_1^2D^2\eta_{r, t}^2 \,(t+1)}{b},\\ 
\ex{}{\eqnorm{g_r^t - \nabla F_{\cD}(w_r^t) }^2\Bigg\vert\cQ_r^{t-1}} & \leq \kappa\frac{(t+1)^2}{b^2}\sum_{z\in B_r^t}\ex{}{\dual{\nabla f(w_r^t, z)  -\nabla F_{\cD}(w_r^t)}^2\Bigg\vert\cQ_r^{t-1}}
\leq 
\frac{2\kappa L_0^2\,(t+1)}{b}
\end{align*}
Moreover, observe that by property~\eqref{eqn:k_equiv_norm} of $\kappa$-regular norms, we have

\begin{align*}
\ex{}{\eqnorm{ \tN_r^t}^2\Bigg\vert\cQ_r^{t-1}}\leq \frac{\kappa}{\kappa_+}\ex{}{\dual{ \tN_r^t}^2\Bigg\vert\cQ_r^{t-1}}
&=\frac{\kappa}{\kappa_+}\ex{}{\norm{ \tN_r^t}_q^2\Bigg\vert\cQ_r^{t-1}}\\
\end{align*}
Note that when $p=q=2$ (i.e., the Euclidean setting), then the above is bounded by $d\tsigma_{r,t}^2$ (in such case, note that $\kappa=\kappa_+=1$). Otherwise (when $1<p<2$), we have
\begin{align*}
\ex{}{\eqnorm{ \tN_r^t}^2\Bigg\vert\cQ_r^{t-1}}\leq \frac{\kappa}{\kappa_+}\ex{}{\dual{ \tN_r^t}^2\Bigg\vert\cQ_r^{t-1}}
&=\frac{\kappa}{\kappa_+}\ex{}{\norm{ \tN_r^t}_q^2\Bigg\vert\cQ_r^{t-1}}\\
&\leq \frac{\kappa}{\kappa_+}d^{\frac{2}{q}}\,\ex{}{\norm{ \tN_r^t}_{\infty}^2\Bigg\vert\cQ_r^{t-1}}\\
&\leq 2\frac{\kappa}{\kappa_+} d^{\frac{2}{q}}\log(d)\,\tsigma^2_{r,t}\\
&=32\frac{\kappa}{\kappa_+}\frac{L_1^2D^2\eta_{r,t}^2(t+1)^2\,d \log(d)\log(1/\delta)}{b^2\varepsilon^2}
\end{align*}
Hence, putting the above together, for any $p\in (1, 2]$, we have 
\begin{align*}
\ex{}{\eqnorm{ \tN_r^t}^2\Bigg\vert\cQ_r^{t-1}}\leq 32\frac{\kappa\tkappa}{\kappa_+}\frac{L_1^2D^2\eta_{r,t}^2(t+1)^2\,d\log(1/\delta)}{b^2\varepsilon^2},
\end{align*}
\cnote{Something important I recently realized: for $\ell_p$-spaces, $\kappa/\kappa_+=\Theta(1)$. I mean , this constant is  $\leq e^2$ (this results from the equivalence of norms $\bar q=\min\{q,2\ln d\}$ and $q$). This should refine our bounds here and in BGN'21.} \rnote{Let's keep this refinement for the arxiv version. For neurips supplementary, it's better to stay consistent with the bounds we claimed in the main.} \cnote{Sure!}
where $\tkappa=1+\log(d)\cdot\ind(p<2)$.

Similarly, we can show
\begin{align*}
\ex{}{\eqnorm{N_r^t}^2\Bigg\vert\cQ_r^{t-1}}&\leq 2\,\frac{\kappa\tkappa}{\kappa_+} d^{\frac{2}{q}}\sigma^2_{r,t}=32\frac{\kappa\tkappa}{\kappa_+}\frac{L_0^2(t+1)^2\,d\log(1/\delta)}{b^2\varepsilon^2}.
\end{align*}
Plugging these bounds in inequality (\ref{eqn:recur1}) and using the setting of $\eta_{r,t}$ in the lemma statement, we arrive at the following recursion: 
\begin{align*}
   \gamma_r^t &\leq \left(1-\frac{1}{\sqrt{t+1}}\right)^2 \gamma_r^{t-1} +8\frac{\kappa\kappa_+(L_0^2+L_1^2D^2)}{b} +128\frac{\kappa\tkappa(L_0^2+L_1^2D^2)(t+1)d\log(1/\delta)}{b^2\varepsilon^2}\\
   &\leq \left(1-\frac{1}{\sqrt{t+1}}\right)^2 \gamma_r^{t-1} +8\frac{\kappa^2(L_0^2+L_1^2D^2)}{b} +128\frac{\kappa\tkappa(L_0^2+L_1^2D^2)(t+1)d\log(1/\delta)}{b^2\varepsilon^2},
\end{align*}
where the last inequality follows from the fact that $\kappa_+\leq \kappa$. Unraveling this recursion similar to what we did in the proof of Lemma~\ref{lem:grad-error-polySFW}, we arrive at 
\begin{align}
   \gamma_r^t &\leq \left(1-\frac{1}{\sqrt{t+1}}\right)^{2t} \gamma_r^0  +\left(8\frac{\kappa^2(L_0^2+L_1^2D^2)}{b} +128\frac{\kappa\tkappa(L_0^2+L_1^2D^2)(t+1)d\log(1/\delta)}{b^2\varepsilon^2}\right) \sqrt{t+1}. \label{eqn:recur_gamma_noisySFW}
\end{align}
Now, we can bound $\gamma_r^0$ via the same approach used before:
\begin{align*}
    \gamma_r^0&=\ex{}{\eqnorm{\frac{1}{b}\sum_{z\in B_r^0}\nabla f(w_r^0,z) - \nabla F_{\cD}(w_r^0) +N_r^0}^2\Bigg\vert\cQ_{r-1}^{2^{r-1}-1}}\\
    &\leq~2\,\ex{}{\eqnorm{\frac{1}{b}\sum_{z\in B_r^0}\nabla f(w_r^0,z) - \nabla F_{\cD}(w_r^0) }^2\Bigg\vert\cQ_{r-1}^{2^{r-1}-1}}+2\,\ex{}{\eqnorm{N_r^0}^2\Bigg\vert\cQ_{r-1}^{2^{r-1}-1}}\\
    &\leq~2\frac{\kappa}{b^2}\sum_{z\in B_r^0}\ex{}{\eqnorm{\nabla f(w_r^0,z) - \nabla F_{\cD}(w_r^0) }^2\Bigg\vert\cQ_{r-1}^{2^{r-1}-1}}+64\frac{\kappa\tkappa}{\kappa_+}\frac{L_0^2 d\log(1/\delta)}{b^2\varepsilon^2}\\
    &\leq 4\frac{\kappa L_0^2}{b}+64\frac{\kappa\tkappa}{\kappa_+}\frac{L_0^2 d\log(1/\delta)}{b^2\varepsilon^2}\\
    &\leq 4\frac{\kappa L_0^2}{b}+64\frac{\kappa\tkappa L_0^2 d\log(1/\delta)}{b^2\varepsilon^2},
\end{align*}
where the last inequality follows from the fact that $\kappa_+\geq 1.$ Plugging this in (\ref{eqn:recur_gamma_noisySFW}), we finally have
\begin{align*}
  \E\left[\eqnorm{\tgrd_r^t-\nabla F_{\cD}(w_r^t)}^2\right]\leq& 64 L_0^2 \left(\frac{\kappa}{b}+\frac{\kappa\tkappa d\log(1/\delta)}{b^2\varepsilon^2}\right)\left(1-\frac{1}{\sqrt{t+1}}\right)^{2t}\\
  &+ 128(L_0^2+L_1^2D^2)\left(\frac{\kappa^2}{b}\,\sqrt{t+1}+\frac{\kappa\tkappa d\log(1/\delta)}{b^2\varepsilon^2}\,(t+1)^{3/2}\right).
\end{align*}
Hence, by property~\eqref{eqn:k_equiv_norm} of $\kappa$-regular norms and using Jensen's inequality together with the subadditivity of the square root, we  conclude
\begin{align*}
  &\E\left[\dual{\tgrd_r^t-\nabla F_{\cD}(w_r^t)}\right]\leq \sqrt{\E\left[\eqnorm{\tgrd_r^t-\nabla F_{\cD}(w_r^t)}^2\right]}\\
  &\leq 8L_0 \left(\sqrt{\frac{\kappa}{b}}+\frac{\sqrt{\kappa\tkappa d\log(1/\delta)}}{b\varepsilon}\right)\left(1-\frac{1}{\sqrt{t+1}}\right)^{t}+16(L_0+L_1D)\left(\frac{\kappa}{\sqrt{b}}\,(t+1)^{1/4}+\frac{\sqrt{\kappa\tkappa d\log(1/\delta)}}{b\varepsilon} \, (t+1)^{3/4}\right).
\end{align*}
\end{proof}

The proof of the convergence guarantee has a similar outline to that of Theorem~\ref{thm:convergence-polySFW} with a few exceptions to account for the additional noise in the gradient estimates $\tgrd_r^t$. \mnote{Moved preceeding statement from after theorem to before theorem}
\paragraph{Proof of Theorem \ref{thm:convergence-noisySFW}}
For any iteration $(r, t)$, using the same derivation approach as in the proof of Theorem~\ref{thm:convergence-polySFW}, we  arrive at the following bound:
$$F_{\cD}(w_r^t)\leq F_\cD(w_r^{t}) +  \eta_{r, t} D \dual{\nabla F_\cD(w_r^t) - \grd_r^t} -\eta_{r, t}\gap_{F_{\cD}}(w_r^t) +\frac{L_1 D^2 \eta_{r,t}^2}{2}$$
Thus, using the bound of Lemma~\ref{lem:grad-error-noisySFW}, the expected stationarity gap of any given iterate $w_r^t$ can be bounded as: 
\begin{align*}
    \E[\gap_{F_{\cD}}(w_r^t)]\leq&~ \frac{\E[F_{\cD}(w_r^t)-F_{\cD}(w_r^{t+1})]}{\eta_{r, t}}+D\,\E\left[\dual{\grd_r^t-\nabla F_{\cD}(w_r^t)}\right]+\frac{L_1 D^2 \eta_{r,t}}{2}\\
    \leq&~ \sqrt{t+1}\left(\E[F_{\cD}(w_r^t)-F_{\cD}(w_r^{t+1})]\right)+ \frac{L_1 D^2}{2\sqrt{t+1}}+ 8 D L_0 \left(\sqrt{\frac{\kappa}{b}}+\frac{\sqrt{d\kappa\tkappa\log(1/\delta)}}{b\varepsilon}\right) \left(1-2^{-r/2}\right)^{t} \\
	&+ 16 D\left( L_1 D + L_0 \right)\left(\frac{\kappa}{\sqrt{b}}(t+1)^{1/4}+\frac{\sqrt{d\kappa\tkappa\log(1/\delta)}}{b\varepsilon}(t+1)^{3/4}\right).
\end{align*}
For any given $r \in \{0, \ldots, R-1\},$ we now sum both sides of the above inequality over $t\in \{0, \ldots, 2^{r}-1\}$ as we did in the proof of Theorem~\ref{thm:convergence-polySFW}.  Let $\Gamma_r\triangleq\sum_{t=0}^{2^{r}-1}\sqrt{t+1}\left(\E[F_{\cD}(w_r^t)-F_{\cD}(w_r^{t+1})]\right).$ Observe that
\begin{align*}
    \sum_{t=0}^{2^r-1}\E[\gap_{F_{\cD}}(w_r^t)]\leq&~ \Gamma_r+\frac{L_1 D^2}{2}\sum_{t=1}^{2^r}\frac{1}{\sqrt{t}}+8 D L_0 \left(\sqrt{\frac{\kappa}{b}}+\frac{\sqrt{d\kappa\tkappa\log(1/\delta)}}{b\varepsilon}\right) \sum_{t=0}^{2^r-1}\left(1-2^{-r/2}\right)^t\\ 
    &+ 16 D\left( L_1 D + L_0 \right)\left(\frac{\kappa}{\sqrt{b}}\sum_{t=1}^{2^r}t^{1/4}+\frac{\sqrt{d\kappa\tkappa\log(1/\delta)}}{b\varepsilon}\sum_{t=1}^{2^r}t^{3/4}\right)\\
    \leq&~ \Gamma_r+L_1D^2 \,2^{r/2}+8 D L_0 \left(\sqrt{\frac{\kappa}{b}}+\frac{\sqrt{d\kappa\tkappa\log(1/\delta)}}{b\varepsilon}\right)2^{r/2}\\
    &+32D\left( L_1 D + L_0 \right)\left(\frac{\kappa}{\sqrt{b}}2^{5r/4}+\frac{\sqrt{d\kappa\tkappa\log(1/\delta)}}{b\varepsilon}2^{7r/4}\right).
\end{align*}
Next, using exactly the same technique we used in the proof of Theorem~\ref{thm:convergence-polySFW}, we can bound $\Gamma_r \leq  3 L_0D \,2^{r/2}.$ Thus, we arrive at 
\begin{align*}
    \sum_{t=0}^{2^r-1}\E[\gap_{F_{\cD}}(w_r^t)]\leq&~3D\left(L_0+L_1D\right)\,2^{r/2}+8 D L_0 \left(\sqrt{\frac{\kappa}{b}}+\frac{\sqrt{d\kappa\tkappa\log(1/\delta)}}{b\varepsilon}\right)2^{r/2}\\
    &+32D\left( L_1 D + L_0 \right)\left(\frac{\kappa}{\sqrt{b}}2^{5r/4}+\frac{\sqrt{d\kappa\tkappa\log(1/\delta)}}{b\varepsilon}2^{7r/4}\right)
\end{align*}

Now, summing over $r\in \{0, \ldots, R-1\}$, we have
\begin{align*}
    \sum_{r=0}^{R-1}\sum_{t=0}^{2^r-1}\E[\gap_{F_{\cD}}(w_r^t)]\leq&~ 9 D\left(L_0+L_1D\right)\, 2^{R/2}+24 D L_0 \left(\sqrt{\frac{\kappa}{b}}+\frac{\sqrt{d\kappa\tkappa\log(1/\delta)}}{b\varepsilon}\right)2^{R/2}\\
    &+48 D\left( L_1 D + L_0 \right)\frac{\kappa}{\sqrt{b}}2^{5R/4}+24 D\left( L_1 D + L_0 \right)\frac{\sqrt{d\kappa\tkappa\log(1/\delta)}}{b\varepsilon}2^{7R/4}.
\end{align*}

Since the output $\widehat{w}$ is uniformly chosen from the set of all $2^R$ iterates, then averaging over all the iterates gives the following (after some algebra similar to what we did in the proof of Theorem~\ref{thm:convergence-polySFW})
\begin{align*}
\E[\gap_{F_{\cD}}(\widehat{w})]=&\frac{1}{2^R}\sum_{r=0}^{R-1}\sum_{t=0}^{2^r-1}\E[\gap_{F_{\cD}}(w_r^t)]\leq~ 9 D(L_0+L_1D)2^{-R/2} + 24 DL_0\left(\sqrt{\frac{\kappa}{b}}+\frac{\sqrt{d\kappa\tkappa\log(1/\delta)}}{b\varepsilon}\right)2^{-R/2}\\&+ 48 D(L_0+L_1D)\frac{\kappa}{\sqrt{b}}2^{R/4}+24 D(L_0+L_1D)\frac{\sqrt{d\kappa\tkappa\log(1/\delta)}}{b\varepsilon}2^{3R/4}.
\end{align*}
Plugging $R=\frac{4}{5}\log\left(\frac{n\varepsilon}{\sqrt{d\tkappa\log(1/\delta)}\,\kappa^{5/3}\log^2(n)}\right),$ we finally get
\begin{align*}
\E[\gap_{F_{\cD}}(\widehat{w})]\leq&~ 9D(L_0+L_1D)\kappa^{2/3}\,\frac{d^{1/5}\,\tkappa^{1/5}\log^{1/5}(1/\delta)\log^{4/5}(n)}{n^{2/5}\varepsilon^{2/5}} \\
+&~ 24DL_0\,\kappa^{2/3}\left(\sqrt{\frac{\kappa\log^2(n)}{n}}+\frac{\sqrt{d\kappa\tkappa\log(1/\delta)}\log^2(n)}{n\varepsilon}\right)\frac{d^{1/5}\,\tkappa^{1/5}\log^{1/5}(1/\delta)\log^{4/5}(n)}{n^{2/5}\varepsilon^{2/5}}\\
&+ 48 D(L_0+L_1D)\kappa^{2/3}\frac{\varepsilon^{1/5}\log^{3/5}(n)}{n^{3/10}\left(d\tkappa\log(1/\delta)\right)^{1/10}}+24D(L_0+L_1D)\frac{d^{1/5}\,\tkappa^{1/5}\log^{1/5}(1/\delta)\log^{4/5}(n)}{\kappa^{1/2}\, n^{2/5}\,\varepsilon^{2/5}}\\
=~& O\left(D(L_0+L_1D)\kappa^{2/3}\left(\frac{\varepsilon^{1/5}\log^{3/5}(n)}{n^{3/10}\left(d\tkappa\log(1/\delta)\right)^{1/10}}+\frac{d^{1/5}\,\tkappa^{1/5}\log^{1/5}(1/\delta)\log^{4/5}(n)}{n^{2/5}\varepsilon^{2/5}}\right)\right),
\end{align*}
Now, observe that the bound above is dominated by the first term when $d\,\tkappa = o\left(\frac{n^{1/3}\varepsilon^2}{\log(1/\delta)\log^{2/3}(n)}\right)$. Moreover, note that the first term is decreasing in $d$. Thus, we can obtain a more refined bound via the following simple argument. When $d\,\tkappa = o\left(\frac{n^{1/3}\varepsilon^2}{\log(1/\delta)\log^{2/3}(n)}\right)$, we embed our optimization problem in higher dimensions; namely, in $d'$ dimensions, where $d'$ satisfies: $d'\big(1+\log(d')\cdot\ind(p<2)\big)=\Theta\left(\frac{n^{1/3}\varepsilon^2}{\log(1/\delta)\log^{2/3}(n)}\right)$. In such case, the bound above (with $d=d'$) becomes $O\left(D(L_0+L_1D)\kappa^{2/3}\,\frac{\log^{2/3}(n)}{n^{1/3}}\right)$. When $d\,\tkappa=\Omega\left(\frac{n^{1/3}\varepsilon^2}{\log(1/\delta)\log^{2/3}(n)}\right)$, the bound above is dominated by the second term. Putting these together, we finally arrive at the claimed bound: 
$$O\left(D(L_0+L_1D)\kappa^{2/3}\left(\frac{\log^{2/3}(n)}{n^{1/3}}+\frac{d^{1/5}\,\tkappa^{1/5}\log^{1/5}(1/\delta)\log^{4/5}(n)}{n^{2/5}\varepsilon^{2/5}}\right)\right).$$

\cnote{Where is the end of proof here? I don't see the square.}\rnote{This is because we prove the theorem later and not right after the theorem statement. I don't think it's a big issue. Btw, I moved the sentence that was here regarding the Euclidean case to right after the theorem statement. I think it should have been there not here.}

\section{Algorithm for Weakly Convex Non-smooth Losses}\label{sec:weakcnvx}

Our final setting is DP stochastic {\em weakly convex} optimization. 
Much of the theory of weakly convex \cnote{correction: weakly convex, instead of weakly smooth} functions is available in  \cite{RockaWets:1998}, but we provide a self-contained exposition in Appendix \ref{sec:app-weak_cvx_basics}.\footnote{Our motivation to reproduce the basic theory stems from the fact that \cite{RockaWets:1998} and much of the literature of weakly convex functions focuses on Euclidean settings, whereas we are interested in more general $\ell_p$ settings.} We recall that a function $f:{\cal W}\mapsto\mathbb{R}$
is $\rho$-weakly convex w.r.t.~$\|\cdot\|$ if for all $0\leq\lambda\leq 1$ and $w,v\in{\cal W}$, 
\begin{equation}\label{eqn:wk_cvx}
    f(\lambda w+(1-\lambda)v) \leq \lambda f(w)+(1-\lambda)f(v)+\frac{\rho\lambda(1-\lambda)}{2}\|w-v\|^2.
\end{equation} 
It is easy to see that any $L_1$-smooth function is indeed $L_1$-weakly convex, so weak convexity encompasses smooth non-convex functions (see Corollary \ref{cor:smooth_implies_wk_cvx} in Appendix \ref{sec:app-weak_cvx_basics}). However, this extension is interesting as it also contains some classes of non-smooth functions. 

\subsection{Proximal-Type Operator and Proximal Near Stationarity} \label{sec:prox_operator}

The next property is crucial for regularization of weakly smooth functions, and it would allow us to make sense of a proximal-type operator in some non-Euclidean norms. 
\mnote{Removed appendix reference}

\begin{prop}\label{prop:reg_wk_cvx}
Let $\|\cdot\|$ be a norm such that $\frac12\|\cdot\|^2$ is $\nu$-strongly convex w.r.t.~$\|\cdot\|$. If $f$ is $\rho$-weakly convex and $\nu\beta\geq\rho$, then the function $w\mapsto f(w)+\frac{\beta}{2}\|w-u\|^2$ is $(\nu\beta-\rho)$-strongly convex w.r.t.~$\|\cdot\|$.
\end{prop}
\begin{proof}
By strong convexity of $\frac12\|\cdot\|^2$ and weak convexity of $f$:
\begin{align*}
\frac{\beta}{2}\|[\lambda w+(1-\lambda)v]-u\|^2
&\leq \lambda \frac{\beta}{2}\|w-u\|^2+(1-\lambda) \frac{\beta}{2}\|v-u\|^2- \frac{\beta \nu\lambda(1-\lambda)}{2}\|w-v\|^2 \\
f(\lambda w+(1-\lambda)v) &\leq \lambda f(w)+(1-\lambda)f(v)+\frac{\rho\lambda(1-\lambda)}{2}\|w-v\|^2
\end{align*}
Adding these inequalities, and using that $\nu\beta\geq\rho$, we conclude the $(\nu\beta-\rho)$-strong convexity of $f(\cdot)+\frac{\beta}{2}\|\cdot-u\|^2$, concluding the proof.
%
\end{proof}

We provide now some useful results regarding a proximal-type mapping for weakly convex functions in normed spaced. This provides a non-Euclidean counterpart to results in \cite{RockaWets:1998,DG:2019,DD:2019}. 
First, given ${\cal W}\subseteq\mathbf{E}$ a closed and convex set, we define the 
proximal-type mapping as: 
\begin{eqnarray}
\prox_{f}^{\beta }(w) &=& \arg\min_{v\in {\cal W}}\big[f(v)+\frac{\beta}{2}\|v-w\|^2\big]. \label{eqn:prox_map}
\end{eqnarray}
Despite the stark similarity with the Euclidean proximal operator, the characterization of proximal points is in general different (due to the formula for the subdifferential of the squared norm), so we need to re-derive some near-stationarity estimates derived in \cite{DD:2019,DG:2019}. 

\begin{lem} \label{lem:prox_estimate}
Let $\|\cdot\|$ be such that $\frac12\|\cdot\|^2$ is differentiable and $\nu$-strongly convex w.r.t.~$\|\cdot\|$, let
$f:\mathbf{E}\mapsto\mathbb{R}$ be a $\rho$-weakly convex subdifferentiable function, ${\cal W}\subseteq\mathbf{E}$ a closed, convex set with diameter $D$, and $\beta>\rho/\nu$.
Then, for any $w\in {\cal W}$, the proximal-type mapping $\hat w=\mbox{prox}_{f}^{\beta}(w)$ (given in \eqref{eqn:prox_map}) 
is well-defined, and moreover there exists $g\in \partial f(\hat w)$ such that
\[ \sup_{v\in {\cal W}}\langle g,\hat w-v\rangle \leq \beta D\|w-\hat w\|.\]
\end{lem}
\begin{proof}
First, notice that the proximal-type mapping can be computed as a solution of the optimization problem
\begin{eqnarray}
\min_{v\in {\cal W}}\big[f(v)+\frac{\beta}{2}\|v-w\|^2\big]. \label{eqn:Moreau}
\end{eqnarray}
By Proposition \ref{prop:reg_wk_cvx}, problem \eqref{eqn:Moreau} is strongly convex, and therefore it has a unique solution; in particular, $\hat w$ is well-defined and unique. 
Next, we use the optimality conditions of constrained convex optimization for problem \eqref{eqn:Moreau}, together with the subdifferential of the sum rule (Theorem \ref{thm:sum_subdiff}), and the chain rule of the convex subdifferential; to conclude that
\begin{equation}\label{eqn:opt_cond_prox_type} 
\Big(\partial f(\hat w)+\beta\|\hat w-w\|\,\partial(\|\cdot\|)(\hat w-w)\Big) \cap -{\cal N}_{\cal W}(\hat{w})\neq \emptyset. 
\end{equation}
First, consider the case where $\hat w=w$, then there exists $g\in \partial f(\hat{w})$ s.t., $\langle g,\, \hat{w}-v\rangle \leq 0$, for all $v\in {\cal W}$, which shows the desired conclusion. In the case $\hat w\neq w$, consider $g\in \partial f(\hat{w})$ and $h\in\partial(\|\cdot\|)(\hat w-w)$ such that by \eqref{eqn:opt_cond_prox_type}, $\langle g+\beta\|\hat w-w\|h,v-\hat w \rangle \geq 0$, for all $v\in {\cal W}$. We first prove that $\|h\|_{\ast}=1$. Indeed, first $\|h\|_{\ast}\leq 1$ since the norm is 1-Lipschitz. The reverse inequality follows from the equality in the Fenchel inequality, when $h$ is a subgradient \cite{HiriartUrruty:2001},
\[ \|\hat w-w\|= \|\hat w-w\|+\chi_{{\cal B}_{\ast}(0,1)}(h)=\langle h,\hat w-w\rangle. \]
Since $\hat w\neq w$, this shows in particular that $\|h\|_{\ast}=1$. We conclude that in this case, $\langle g,\hat w-v \rangle \leq \beta D\|w-\hat w\|$, for all $v\in {\cal W}$, which concludes the proof. \cnote{Made a small correction $p\mapsto h$, as it was unclear to which $p$ I was referring to.}
\end{proof}

The previous lemma is the key insight on the accuracy guarantee and algorithms we will use for  stochastic weakly convex optimization. First, note that in the weakly convex setting it is unlikely to find points with small norm of the gradient or small stationarity gap; however, we will settle for points $w\in{\cal W}$ which are $\vartheta$-{\em close to a nearly-stationary point} \cite{DD:2019,DG:2019}, i.e., that satisfies
\begin{equation}\label{eqn:close_near_stationary}
    (\exists \hat w\in{\cal W})(\exists g\in \partial f(\hat w)): \quad \|w-\hat w\|\leq \vartheta\quad \mbox{ and }\quad \sup_{v\in{\cal W}}\langle g,\hat w-v\rangle\leq \vartheta.
\end{equation}
Above, $\vartheta\geq 0$ is the accuracy parameter. This accuracy measure states that $w$ is at distance at most $\vartheta$ from a $\vartheta$-nearly stationary point. 
It is then apparent how the proximal-type operator can certify \eqref{eqn:close_near_stationary}. For convenience, we define a notion of efficiency in weakly-convex DP-SO, particularly geared towards algorithms that certify close to near stationarity via the proximal-type mapping.

\begin{defn}[Proximal Near Stationarity] \label{def:prox_near_stat} A randomized algorithm  
${\cal A}:{\cal Z}^n\mapsto\mathbf{E}$, for the stochastic optimization problem $\min_{w\in {\cal W}}F_{\cal D}(w)$, achieves $(\vartheta,\beta)$-proximal  
near stationarity if
\begin{equation} \label{eqn:prox_near_stat}
\mathbb{E}_{S\sim{\cal D}^n,{\cal A}}\big[\|\prox_{F_{\cal D}}^{\beta}({\cal A}(S))-{\cal A}(S)\| \big] \leq \vartheta/\max\{1,\beta D\}. 
\end{equation}
\end{defn}
Notice the maximum in the denominator is a normalizing factor, inspired by Lemma \ref{lem:prox_estimate}.
Note further that, by Lemma  \ref{lem:prox_estimate}, an algorithm with proximal near stationarity ensures closeness to nearly stationary points through its proximal-type mapping: namely, if ${\cal A}$ satisfies Definition \ref{def:prox_near_stat}, then 
\begin{equation*} 
\mathbb{E}_{S\sim{\cal D}^n,{\cal A}}\big[\|\prox_{F_{\cal D}}^{\beta}({\cal A}(S))-{\cal A}(S)\| \big] \leq \vartheta \qquad\mbox{ and }\qquad 
\mathbb{E}_{S\sim{\cal D}^n,{\cal A}} \big[\gap_{F_{\cal D}}\big(\prox_{F_{\cal D}}^{\beta}({\cal A}(S))\big)\big] \leq \vartheta.
\end{equation*}
In the above, some technical caution must be taken to define the gap function in the stochastic non-smooth case, which we defer to Appendix \ref{sec:PNS_conseq}.
Although not defined under this name, this is precisely the certificate achieved in weakly-convex SO in recent literature \cite{DG:2019,DD:2019}.

\subsection{Proximally Guided Private Stochastic Mirror Descent} \label{sec:prox_guided_DP_SMD}

\begin{algorithm}[!h]
\caption{Proximally Guided Private Stochastic Mirror Descent}
\begin{algorithmic}[1]

\REQUIRE Private dataset $S=(z_1,\ldots,z_n)\in{\cal Z}^n$, number of rounds $R$, $\beta>0$ regularization parameter

\STATE Let $\overline{p}=\max\{p,1+1/\log d\}$, and choose initialization $w_1\in {\cal W}$

\FOR{$r =1$ to $R$}
\STATE Extract batch $S_r$ from $S\setminus \bigcup_{l<r} S_r$ of size, $n_r=n/R$

\STATE Let $w_{r+1}$ the the output of $\Asc$ on data $S_r$ for the objective
    \begin{equation} \label{eqn:proximal_weak_cvx_step} \min_{w\in{\cal W}}F_r(w):=\big\{F_{\cal D}(w)+\frac{\beta}{2}\|w-w_r\|_{\overline{p}}^2\big\}
    \end{equation}

\ENDFOR

\STATE Output: Output $\overline{w}^R$, chosen uniformly at random from $(w_r)_{r\in[R]}$.
\end{algorithmic}
\label{Alg:PGPSMD}
\end{algorithm}

Now we provide an algorithm (Algorithm \ref{Alg:PGPSMD}) \mnote{Added alg reference} for DP-SO with weakly convex losses that certifies proximal near stationarity. This algorithm is inspired by the {\em proximally guided stochastic subgradient method} of Davis and Grimmer \cite{DG:2019}, where the proximal subproblems are solved using an optimal algorithm for DP-SCO in the strongly convex case, proposed in \cite{Asi:2021}, that we call $\Asc$ (see Theorem \ref{thm:sc_dp_sco} below). 
Our algorithm works in rounds $r=1,\ldots, R$, and at each round the  proximal-type mapping subproblem
\[ \min_{w\in {\cal W}} F_r(w)=\Big\{ F_{\cal D}(w)+\frac{\beta}{2}\|w-w_r\|_{\bar p}^2 \Big\}, \]
is approximately solved using a separate minibatch of size $n/R$ with algorithm $\Asc$. The $\bar{p}$ used in the subproblem norm is chosen as $\bar p =\max\{p,1+1/\log d\}$, in order to control the strong convexity.
Finally, the output is chosen uniformly at random from the iterates. 
\mnote{Removed appendix reference}



\begin{thm}[Theorem~8 in \cite{Asi:2021}] \label{thm:sc_dp_sco}
Consider the $\ell_p$ setting of $\lambda$-strongly convex stochastic optimization, where $1\leq p\leq 2$.
There exists an $(\varepsilon,\delta)$-differentially private algorithm $\Asc$ with excess risk
$$ O\Big(\frac{L_0^2}{\lambda}\Big[\frac{\kappa}{n}+\frac{\tilde{\kappa}\kappa^2 d\log(1/\delta)}{ n^2\varepsilon^2}\Big]  \Big), $$
where $\kappa=\min\{1/(p-1),\log d\}$ and $\tilde{\kappa} = 1+ \log{d}\cdot\ind(p<2)$. 
This algorithm runs in time $O(\log n\cdot\log\log n\cdot\min\{n^{3/2}\sqrt{\log d},n^2\varepsilon/\sqrt{d} \})$.
\end{thm}
We note in passing that \cite[Theorem~8]{Asi:2021} is stated only for the $\ell_1$-setting; however, since their mirror descent algorithm and reduction to the strongly convex case works more generally, the theorem stated above is a more general version of their result. 

We now state and prove our main result in this section.


\begin{thm} \label{thm:nsmth_ncnvx}
Consider the $\ell_p$ setting of $\rho$-weakly convex stochastic  optimization, where $1\leq p\leq 2$. Let $\kappa=\min\{1/(p-1),\log d\}$, $\tilde{\kappa} =1+ \log{d}\cdot\ind(p<2)$, and $\beta=2\rho\kappa$. Suppose that  $n d\geq \rho D/L_0$. 
Then the output of the Proximally Guided Private Stochastic Mirror Descent (Algorithm \ref{Alg:PGPSMD}) is $(\varepsilon,\delta)$-DP, and for \linebreak $R=\Big\lfloor\min\Big\{\sqrt{\frac{ n D \rho}{ \kappa L_0}}, \frac{1}{(\tilde{\kappa}\kappa^2)^{1/3}}\big(\frac{D(n\varepsilon)^2\rho}{L_0d\log(1/\delta)}\big)^{1/3} \Big\}\Big\rfloor$, it is
guaranteed to provide a $(\vartheta,\beta)$-proximal nearly stationary point, with 
\begin{equation} \label{eqn:CTNSP_DP_prox}
\vartheta= \frac{\max\{1,2\rho D\kappa\}}{\sqrt{\rho}}O\Big(\Big(\frac{L_0^{3}D\kappa}{n\rho}\Big)^{1/4} + (\tilde{\kappa}\kappa^2)^{1/6}(L_0^2D)^{1/3}\Big(\frac{d\log(1/\delta)}{(n\varepsilon)^2\rho}\Big)^{1/6}  \Big).
\end{equation}
The running time of this algorithm is upper bounded by $O\left(\log n \cdot\log\log n\cdot\min\left(n^{3/2}\sqrt{\log d},n^2\varepsilon/\sqrt{d}\right)\right).$ 
\cnote{Moved this sentence up here. It was cutoff by the proof.}
\begin{proof}
The privacy of this algorithm is certified by parallel composition and the privacy guarantees of $\Asc$. For the accuracy,  
first consider the case $p\geq 1+1/\log d$. Here, recall that $w\mapsto\frac{1}{2}\|w-\bar w\|_p^2$ is a $1/\kappa$-strongly convex function w.r.t.~$\|\cdot\|_p$ \cite{Beck:2017}, so we can choose $\nu=1/\kappa=(p-1)$ as the strong convexity parameter. 
Let $\hat w_r=\prox_{F_{\cal D}}^{\beta}(w_r)$ be the optimal solution to problem \eqref{eqn:proximal_weak_cvx_step}. Our goal now is to show that $\overline{w}^R$ is $\vartheta$-proximal nearly stationary. First, by Proposition \ref{prop:reg_wk_cvx}, $F_r$ is $(\beta/\kappa-\rho)$-strongly convex w.r.t.~$\|\cdot\|_p$. Since $(\beta/\kappa-\rho)=\rho$, we have by Theorem \ref{thm:sc_dp_sco} that
for all $r=1,\ldots,R$, 
\begin{equation} \label{eqn:guarantee_Asc} \mathbb{E}\big[F_r(w_{r+1})-F_r(\hat w_{r})\big] =
 O\Big(\frac{L_0^2}{\rho}\Big[\frac{\kappa}{ n_r}+\frac{\tilde{\kappa}\kappa^2 d\log(1/\delta)}{ n_r^2\varepsilon^2}\Big]  \Big). 
\end{equation}
By strong convexity of $F_r$, we have almost surely:
\begin{align} 
   F_{\cal D}(w_r) = F_{r}(w_r) &\geq F_{r}(\hat w_r) +\frac{\rho}{2}\|\hat w_r-w_r\|_p^2.   \label{eqn:dist_prox_risk} 
\end{align}
Hence, using \eqref{eqn:guarantee_Asc} and \eqref{eqn:dist_prox_risk}, we get
\begin{align*}
    \mathbb{E}\Big[F_{\cal D}(w_{r+1})+\frac{\beta}{2}\|w_{r+1}- w_r\|_p^2\Big] 
    &= \mathbb{E}[F_r(w_{r+1})]
    \leq \mathbb{E}[F_r(\hat w_r)]+O\Big(\frac{L_0^2}{\rho}\Big[\frac{\kappa}{n_r}+\frac{\tilde{\kappa}\kappa^2 d\log(1/\delta)}{ n_r^2\varepsilon^2}\Big]  \Big) \\
    &= \mathbb{E}\Big[F_{\cal D}(w_r)-\frac{\rho}{2}\|\hat w_r-w_r\|_p^2\Big]+ O\Big(\frac{L_0^2}{\rho}\Big[\frac{\kappa }{ n_r}+\frac{\tilde{\kappa}\kappa^2 d\log(1/\delta)}{ n_r^2\varepsilon^2}\Big]  \Big),
\end{align*}
and summing from $r=1,\ldots,R$, we obtain 
\begin{align*}
\frac1R\sum_{r=1}^R\mathbb{E}\|\hat w_r-w_r\|^2
&\leq \frac{2}{R\rho}\Big[\mathbb{E}[F(w_1)-F(w_{R+1})]+O\Big(\sum_{r=1}^R\frac{L_0^2}{\rho}\Big[\frac{\kappa}{ n_r}+\frac{\tilde{\kappa}\kappa^2 d\log(1/\delta)}{ n_r^2\varepsilon^2}\Big]\Big)\Big] \\
&= O\Big(\frac{1}{\rho}\Big\{\frac{L_0D}{R} +
\frac{L_0^2}{\rho}\Big[\kappa\frac{R}{ n}+\frac{\tilde{\kappa}\kappa^2 d\log(1/\delta)}{\varepsilon^2}\frac{R^2}{n^2}\Big]\Big\}\Big).
\end{align*}
Now we use that $R=\Big\lfloor\min\Big\{\sqrt{\frac{ n D \rho}{ \kappa L_0}}, \frac{1}{(\tilde{\kappa}\kappa^2)^{1/3}}\big(\frac{D(n\varepsilon)^2\rho}{L_0d\log(1/\delta)}\big)^{1/3} \Big\}\Big\rfloor$, which is at most $n$ by the assumption $n d\geq \rho D/L_0$. Then,
\begin{align*}
 \mathbb{E}\big[\|\prox_{F_{\cal D}}(\overline{w}^R)-\overline{w}^R\|_p^2\big]
&=  \frac1R\sum_{r=1}^R\mathbb{E}\big[\|\hat w_r-w_r\|_p^2\big]
=O\left(\frac{1}{\rho}\Big[\Big(\frac{L_0^{3}D\kappa}{n\rho}\Big)^{1/2} +(\tilde{\kappa}\kappa^2)^{1/3}(L_0^2D)^{2/3}\Big(\frac{d\log(1/\delta)}{(n\varepsilon)^2\rho} \Big)^{1/3}\Big]\right).  
\end{align*}
Finally, by the Jensen inequality, we have that
\[ 
\mathbb{E}\big[\max\{1,\beta D\}\|\prox_{F_{\cal D}}(\overline{w}^R)-\overline{w}^R\|_p\big]
\leq \frac{\max\{1,2\rho D\kappa\}}{\sqrt{\rho}}O\Big(\Big(\frac{L_0^{3}D\kappa}{n\rho}\Big)^{1/4} + (\tilde{\kappa}\kappa^2)^{1/6}(L_0^2D)^{1/3}\Big(\frac{d\log(1/\delta)}{(n\varepsilon)^2\rho}\Big)^{1/6}  \Big).
\]

Next, in the case $1\leq p<1+1/\log d$, we can use that $\|\cdot\|_{\bar p}$ and $\|\cdot\|_p$ are equivalent with a constant factor (recall that here $\bar p=1+1/\log d$). Using then $\|\cdot\|_{\bar p}$ in the algorithm and argument above clearly leads to the same conclusion with $\kappa=\log d$. Finally, the running time upper bound follows by Theorem \ref{thm:sc_dp_sco}.
%
\end{proof}
\end{thm}

\begin{rem}
Some comments are in order. 
First, the bound from eqn.~\eqref{eqn:CTNSP_DP_prox} takes the particular form for $p=1$ and $p=2$, respectively,
\[ 
\vartheta = \left\{
\begin{array}{ll}
\frac{\max\{1,2\rho D\log d\}}{\sqrt{\rho}}O\left(\Big(\frac{L_0^{3}D\log d}{n\rho}\Big)^{1/4} + \sqrt{\log d}(L_0^2 D)^{1/3}\Big(\frac{d\log(1/\delta)}{(n\varepsilon)^2\rho}\Big)^{1/6}  \right) & p=1\\
\frac{\max\{1,2\rho D\}}{\sqrt{\rho}}O\left(\Big(\frac{L_0^{3}D}{n\rho}\Big)^{1/4} +(L_0^2 D)^{1/3}\Big(\frac{d\log(1/\delta)}{(n\varepsilon)^2\rho}\Big)^{1/6}  \right) & p=2.
\end{array}
\right.
\]
Second, the upper bound in running time can be further refined, taking into account the precise value of $R$. We omit the resulting bound, only for simplicity. 
Finally, we note that the accuracy of our algorithm can be further refined, if one considers the initial optimality gap, $\Delta_F=F_{\cal D}(w_1)-F_{\cal D}(w^{\ast})$, instead of the crude upper bound $\Delta_F\leq L_0D$. We make this choice only for simplicity, and to keep consistency with the previous sections.
\end{rem}

\section*{Acknowledgements}\label{sec:ack}
RB's and MM's research is supported by NSF Award AF-1908281, Google Faculty Research Award, and the OSU faculty start-up support. CG's research is partially supported by INRIA through the INRIA Associate Teams project and FONDECYT 1210362 project.

\bibliographystyle{alpha} 

\bibliography{ref1,ref2,ref3}

\newcommand{\etalchar}[1]{$^{#1}$}
\begin{thebibliography}{DKM{\etalchar{+}}06}

\bibitem[ABRW12]{Agarwal:2012}
Alekh Agarwal, Peter~L. Bartlett, Pradeep Ravikumar, and Martin~J. Wainwright.
\newblock Information-theoretic lower bounds on the oracle complexity of
  stochastic convex optimization.
\newblock {\em {IEEE} Trans. Inf. Theory}, 58(5):3235--3249, 2012.

\bibitem[ACD{\etalchar{+}}19]{ACDFSW:2019}
Yossi Arjevani, Yair Carmon, John~C. Duchi, Dylan~J. Foster, Nathan Srebro, and
  Blake~E. Woodworth.
\newblock Lower bounds for non-convex stochastic optimization.
\newblock {\em CoRR}, abs/1912.02365, 2019.

\bibitem[AFKT21]{Asi:2021}
Hilal Asi, Vitaly Feldman, Tomer Koren, and Kunal Talwar.
\newblock Private stochastic convex optimization: Optimal rates in l1 geometry.
\newblock {\em CoRR}, abs/2103.01516, 2021.

\bibitem[Bec17]{Beck:2017}
Amir Beck.
\newblock {\em First-order methods in optimization}.
\newblock SIAM, 2017.

\bibitem[BFGT20]{bassily2020stability}
Raef Bassily, Vitaly Feldman, Crist{\'{o}}bal Guzm{\'{a}}n, and Kunal Talwar.
\newblock Stability of stochastic gradient descent on nonsmooth convex losses.
\newblock In {\em Advances in Neural Information Processing Systems 33, NeurIPS
  2020, December 6-12, 2020, virtual}, 2020.

\bibitem[BFTT19]{BFTT:19}
Raef Bassily, Vitaly Feldman, Kunal Talwar, and Abhradeep Thakurta.
\newblock Private stochastic convex optimization with optimal rates.
\newblock In H.~Wallach, H.~Larochelle, A.~Beygelzimer, F.~d~Alch\'{e}-Buc,
  E.~Fox, and R.~Garnett, editors, {\em Advances in Neural Information
  Processing Systems}, volume~32. Curran Associates, Inc., 2019.

\bibitem[BGN21]{BGN:2021}
Raef Bassily, Crist{\'{o}}bal Guzm{\'{a}}n, and Anupama Nandi.
\newblock Non-euclidean differentially private stochastic convex optimization.
\newblock {\em ArXiv}, abs/2103.01278, 2021.

\bibitem[BLST10]{bhaskar2010discovering}
Raghav Bhaskar, Srivatsan Laxman, Adam Smith, and Abhradeep Thakurta.
\newblock Discovering frequent patterns in sensitive data.
\newblock In {\em Proceedings of the 16th ACM SIGKDD international conference
  on Knowledge discovery and data mining}, pages 503--512, 2010.

\bibitem[BST14]{BST}
Raef Bassily, Adam Smith, and Abhradeep Thakurta.
\newblock Private empirical risk minimization: Efficient algorithms and tight
  error bounds.
\newblock In {\em IEEE 55th Annual Symposium on Foundations of Computer Science
  (FOCS 2014). (arXiv preprint arXiv:1405.7085)}, pages 464--473. 2014.

\bibitem[Can11]{Can11}
Emmanuel Candes.
\newblock Mathematical optimization.
\newblock {\em Lec. notes: MATH 301}, Lec, notes: MATH 301, 2011.

\bibitem[CMS11]{CMS}
Kamalika Chaudhuri, Claire Monteleoni, and Anand~D Sarwate.
\newblock Differentially private empirical risk minimization.
\newblock {\em Journal of Machine Learning Research}, 12(Mar):1069--1109, 2011.

\bibitem[DD19]{DD:2019}
Damek Davis and Dmitriy Drusvyatskiy.
\newblock Stochastic model-based minimization of weakly convex functions.
\newblock {\em SIAM Journal on Optimization}, 29(1):207--239, 2019.

\bibitem[DG19]{DG:2019}
Damek Davis and Benjamin Grimmer.
\newblock Proximally guided stochastic subgradient method for nonsmooth,
  nonconvex problems.
\newblock {\em {SIAM} J. Optim.}, 29(3):1908--1930, 2019.

\bibitem[DKM{\etalchar{+}}06]{DKMMN06}
Cynthia Dwork, Krishnaram Kenthapadi, Frank McSherry, Ilya Mironov, and Moni
  Naor.
\newblock Our data, ourselves: Privacy via distributed noise generation.
\newblock In {\em EUROCRYPT}, 2006.

\bibitem[DR14]{DR14}
Cynthia Dwork and Aaron Roth.
\newblock The algorithmic foundations of differential privacy.
\newblock {\em Foundations and Trends{\textregistered} in Theoretical Computer
  Science}, 9(3--4):211--407, 2014.

\bibitem[DRV10]{DRV10}
Cynthia Dwork, Guy~N. Rothblum, and Salil~P. Vadhan.
\newblock Boosting and differential privacy.
\newblock In {\em FOCS}, 2010.

\bibitem[FGV17]{FGV:16}
Vitaly Feldman, Cristobal Guzman, and Santosh Vempala.
\newblock Statistical query algorithms for mean vector estimation and
  stochastic convex optimization.
\newblock In {\em Proceedings of the Twenty-Eighth Annual ACM-SIAM Symposium on
  Discrete Algorithms}, SODA '17, page 1265–1277, USA, 2017. Society for
  Industrial and Applied Mathematics.

\bibitem[FKT20]{FKT:2020}
Vitaly Feldman, Tomer Koren, and Kunal Talwar.
\newblock Private {Stochastic} {Convex} {Optimization}: {Optimal} {Rates} in
  {Linear} {Time}.
\newblock page~22, 2020.

\bibitem[HKMS20]{Hassani:2020}
Hamed Hassani, Amin Karbasi, Aryan Mokhtari, and Zebang Shen.
\newblock Stochastic conditional gradient++: (non)convex minimization and
  continuous submodular maximization.
\newblock {\em {SIAM} J. Optim.}, 30(4):3315--3344, 2020.

\bibitem[HRS16]{hardt-recht-singer'16}
M.~Hardt, B.~Recht, and Y.~Singer.
\newblock Train faster, generalize better: stability of stochastic gradient
  descent.
\newblock In {\em ICML}, 2016.

\bibitem[HUL01]{HiriartUrruty:2001}
J.-B. Hiriart-Urruty and C.~Lemar{\'e}chal.
\newblock {\em Convex Analysis And Minimization Algorithms}, volume I and II.
\newblock Springer, 2001.

\bibitem[JKT12]{jain2012differentially}
Prateek Jain, Pravesh Kothari, and Abhradeep Thakurta.
\newblock Differentially private online learning.
\newblock In {\em 25th Annual Conference on Learning Theory (COLT)}, pages
  24.1--24.34, 2012.

\bibitem[JN08]{Juditsky:2008}
Anatoli Juditsky and Arkadi Nemirovski.
\newblock Large deviations of vector-valued martingales in 2-smooth normed
  spaces.
\newblock Rapport de recherche hal-00318071, HAL, 2008.

\bibitem[JT14]{JTOpt13}
Prateek Jain and Abhradeep Thakurta.
\newblock (near) dimension independent risk bounds for differentially private
  learning.
\newblock In {\em ICML}, 2014.

\bibitem[KLL21]{KLL:2021}
Janardhan Kulkarni, Yin~Tat Lee, and Daogao Liu.
\newblock Private {Non}-smooth {Empirical} {Risk} {Minimization} and
  {Stochastic} {Convex} {Optimization} in {Subquadratic} {Steps}.
\newblock {\em arXiv:2103.15352 [cs, stat]}, March 2021.
\newblock arXiv: 2103.15352.

\bibitem[KST12]{kifer2012private}
Daniel Kifer, Adam Smith, and Abhradeep Thakurta.
\newblock Private convex empirical risk minimization and high-dimensional
  regression.
\newblock In {\em Conference on Learning Theory}, pages 25--1, 2012.

\bibitem[Mor65]{Moreau:1965}
Jean~Jacques Moreau.
\newblock Proximit\'e et dualit\'e dans un espace hilbertien.
\newblock {\em Bulletin de la Soci\'et\'e Math\'ematique de France},
  93:273--299, 1965.

\bibitem[Nem95]{Nemirovski95_notes}
A~Nemirovski.
\newblock Information based complxity of convex programming.
\newblock 1995.

\bibitem[Nes05]{Nes05}
Yu~Nesterov.
\newblock Smooth minimization of non-smooth functions.
\newblock {\em Math. Program.}, 103(1):127–152, May 2005.

\bibitem[NY83]{NY82}
A.S. Nemirovsky and D.B. Yudin.
\newblock {\em Problem Complexity and Method Efficiency in Optimization}.
\newblock A Wiley-Interscience publication. Wiley, 1983.

\bibitem[RW98]{RockaWets:1998}
{R. Tyrrell} Rockafellar and Roger J.-B. Wets.
\newblock {\em Variational Analysis}.
\newblock Springer Verlag, Heidelberg, Berlin, New York, 1998.

\bibitem[SSBD14]{shalev2014understanding}
Shai Shalev-Shwartz and Shai Ben-David.
\newblock {\em Understanding machine learning: From theory to algorithms}.
\newblock Cambridge university press, 2014.

\bibitem[SSTT21]{SSTT:21}
Shuang Song, Thomas Steinke, Om~Thakkar, and Abhradeep Thakurta.
\newblock Evading the curse of dimensionality in unconstrained private glms.
\newblock In Arindam Banerjee and Kenji Fukumizu, editors, {\em Proceedings of
  The 24th International Conference on Artificial Intelligence and Statistics},
  volume 130 of {\em Proceedings of Machine Learning Research}, pages
  2638--2646. PMLR, 13--15 Apr 2021.

\bibitem[TTZ15]{TTZ15a}
Kunal Talwar, Abhradeep Thakurta, and Li~Zhang.
\newblock Nearly optimal private lasso.
\newblock In {\em NIPS}, 2015.

\bibitem[TTZ16]{talwar_private_2016}
Kunal Talwar, Abhradeep Thakurta, and Li~Zhang.
\newblock Private {Empirical} {Risk} {Minimization} {Beyond} the {Worst}
  {Case}: {The} {Effect} of the {Constraint} {Set} {Geometry}.
\newblock {\em arXiv:1411.5417 [cs, stat]}, November 2016.
\newblock arXiv: 1411.5417.

\bibitem[WCX19]{wang19c}
Di~Wang, Changyou Chen, and Jinhui Xu.
\newblock Differentially private empirical risk minimization with non-convex
  loss functions.
\newblock In Kamalika Chaudhuri and Ruslan Salakhutdinov, editors, {\em
  Proceedings of the 36th International Conference on Machine Learning},
  volume~97 of {\em Proceedings of Machine Learning Research}, pages
  6526--6535. PMLR, 09--15 Jun 2019.

\bibitem[WJEG19]{WJEG:19}
Lingxiao Wang, Bargav Jayaraman, David Evans, and Quanquan Gu.
\newblock Efficient privacy-preserving nonconvex optimization.
\newblock {\em CoRR}, abs/1910.13659, 2019.

\bibitem[WX19]{WX:19}
Di~Wang and Jinhui Xu.
\newblock Differentially private empirical risk minimization with smooth
  non-convex loss functions: A non-stationary view.
\newblock In {\em Proceedings of the AAAI Conference on Artificial
  Intelligence}, volume~33, pages 1182--1189, 2019.

\bibitem[WYX17]{wang_differentially_2017}
Di~Wang, Minwei Ye, and Jinhui Xu.
\newblock Differentially {Private} {Empirical} {Risk} {Minimization}
  {Revisited}: {Faster} and {More} {General}.
\newblock In I.~Guyon, U.~V. Luxburg, S.~Bengio, H.~Wallach, R.~Fergus,
  S.~Vishwanathan, and R.~Garnett, editors, {\em Advances in {Neural}
  {Information} {Processing} {Systems}}, volume~30. Curran Associates, Inc.,
  2017.

\bibitem[ZCH{\etalchar{+}}20]{ZCHWB:20}
Yingxue Zhou, Xiangyi Chen, Mingyi Hong, Zhiwei~Steven Wu, and Arindam
  Banerjee.
\newblock Private stochastic non-convex optimization: Adaptive algorithms and
  tighter generalization bounds.
\newblock {\em CoRR}, abs/2006.13501, 2020.

\bibitem[ZSM{\etalchar{+}}20]{zhang2020one}
Mingrui Zhang, Zebang Shen, Aryan Mokhtari, Hamed Hassani, and Amin Karbasi.
\newblock One sample stochastic frank-wolfe.
\newblock In {\em International Conference on Artificial Intelligence and
  Statistics}, pages 4012--4023. PMLR, 2020.

\bibitem[ZZMW17]{zhang_efficient}
Jiaqi Zhang, Kai Zheng, Wenlong Mou, and Liwei Wang.
\newblock Efficient private erm for smooth objectives.
\newblock In {\em Proceedings of the 26th International Joint Conference on
  Artificial Intelligence}, IJCAI'17, page 3922–3928. AAAI Press, 2017.

\end{thebibliography}

\appendix
\section{Missing Details of Section~\ref{sec:cnvx_glms}} \label{sec:app_glm}

\subsection{Proof of lemma \ref{lem:f_smoothness}}
The Lipschitzness guarantee follows straightforwardly from Lemma \ref{lem:moreau}. 
For the smoothness guarantee, note that $\nabla f_\beta(w,(x,y))=\dellb(\langle w,x \rangle)x$. Since $\ellb$ is $\beta$-smooth, for any $w,w'\in\cW$ we have
\begin{align*}
    \|\nabla f_{\beta}(w,(x,y)) - \nabla f_{z,\beta}(w',(x,y))\|_{*} & = \|\dellb(\langle w,x \rangle)x - \dellb(\langle w',x \rangle)x\|_{*} \\
    & = \|x\|_{*} \cdot |\dellb(\langle w,x \rangle) - \dellb (\langle w',x \rangle)|  \\
    & \leq \|x\|_{*} \beta|\langle w,x \rangle - \langle w',x \rangle| \\
    & \leq \|x\|_{*} ^2\beta\|w-w'\|,
\end{align*}
where the last step follows from the definition of the dual norm. 
For the accuracy, by the guarantees of the Moreau envelope of $\elly$ it holds that for all $w\in\re^d$ and $(x,y)\in\cX\times\re$ that
\begin{align*}
    |f(w,(x,y))-f_{\beta}(w,(x,y))| & = |\elly(\langle w,x \rangle) - \ellb(\langle w,x \rangle )| \\
    & \leq \frac{L_0^2}{2\beta}.
\end{align*}

\section{Missing Details of Section~\ref{sec:weakcnvx}}\label{app:weakcnvx}

For this section, we will occasionally require the use of indicator functions. Given a closed convex set ${\cal W}$, we define the (convex) indicator function as
\[
\chi_{\cal W}(w)=\left\{ 
\begin{array}{rl}
0 & w\in {\cal W}\\
+\infty & w\notin {\cal W}.
\end{array}
\right.
\]
Also recall the definition of the normal cone of ${\cal W}$ at point $\overline{w}\in {\cal W}$, ${\cal N}_{\cal W}(\overline{w})=\{p\in {\cal W}:\,\langle p,w-\overline{w}\rangle \leq 0\,\, \forall w\in{\cal W}\}.$ The normal cone is the subdifferential of the indicator function: ${\cal N}_{\cal W}(w)=\partial \chi_{\cal W}(w)$.  

\subsection{Background Information on Weakly Convex Functions and their Subdifferentials} \label{sec:app-weak_cvx_basics}

\begin{defn}
We say that a function $f:{\cal W}\mapsto\mathbb{R}$ is $\rho$-weakly convex w.r.t.~norm $\|\cdot\|$ if for all $0\leq\lambda\leq 1$ and $w,v\in{\cal W}$, we have
\begin{equation*}
    f(\lambda w+(1-\lambda)v) \leq \lambda f(w)+(1-\lambda)f(v)+\frac{\rho\lambda(1-\lambda)}{2}\|w-v\|^2.
\end{equation*} 
\end{defn}

For nonconvex functions, defining the subdifferential can be done in a local fashion. 

\begin{defn}
Let $f:\mathbf{E}\mapsto\mathbb{R}$. We define the {\em (regular) subdifferential} 
of $f$ at point $w\in\mathbf{E}$, denoted $\partial f(w)$, as the set of vectors $g\in \mathbf{E}$ such that
$$ \liminf_{v\to w,v\neq w} \frac{f(v)-f(w)-\langle g,v-w\rangle}{\|v-w\|}\geq 0. $$
We say that $f$ is subdifferentiable at $w$ if $\partial f(w)\neq \emptyset$. We will say $f$ is subdifferentiable if it is subdifferentiable at every point.
\end{defn}

We will need a characterization of the regular subdifferential in terms of directional derivatives. We recall the definition of the directional derivative of a function $f$ at point $w$ in direction $\dir$:
$$ f^{\prime}(x;\dir) :=\liminf_{\varepsilon\to0, c\to \dir} \frac{f(w+\varepsilon \dir)-f(w)}{\varepsilon}. $$

\begin{prop}[Regular subdifferential and directional derivatives] \label{prop:loc_subdiff}
Let $f:\mathbf{E}\mapsto\mathbb{R}$ be a Lipschitz function which is subdifferentiable at $w$, then 
$$ \partial f(w)=\{g\in\mathbf{E}:\, \langle g,\dir\rangle \leq f^{\prime}(w;\dir)\,\, \forall \dir\in\mathbf{E} \}.$$
\end{prop}

\begin{proof}
Let $L_0$ be the Lipschitz constant of $f$ w.r.t.~$\|\cdot\|$.
We prove both inclusions. First ($\subseteq$), if $g\in \partial f(w)$, then let $\dir\in\mathbf{E}\setminus\{0\}$. Using the definition of subdifferential for $w$ and $v=w+\varepsilon c$ (where $\varepsilon\to0$  and $c\to \dir$), we get
\begin{align*}
\liminf_{\varepsilon\to0,c\to \dir}\frac{f(w+\varepsilon c)-f(w)}{\varepsilon\|c\|}-\frac{\langle g,c\rangle}{\|c\|}\geq 0
\end{align*}
Taking first the limit $c\to \dir$ and then $\varepsilon\to0$, we get $f^{\prime}(w;\dir)\geq\langle g,\dir\rangle$, concluding the desired inclusion.

For the reverse inclusion ($\supseteq$), let $g\in\mathbf{E}$ be s.t.~$\langle g,\dir\rangle \leq f^{\prime}(w;\dir)$, for all $\dir\in \mathbf{E}$. Now let $v\to w$, and consider any $\dir\in \mathbf{E}$ accumulation point of $(v-w)/\|v-w\|$ (they exist by compactness of the unit sphere). Next, let $\varepsilon=\|v-w\|$, and notice that $\varepsilon\to0$. Then
\begin{eqnarray*}
f(v) &=& f(w)+[f(v)-f(w+\varepsilon \dir)]+[f(w+\varepsilon \dir)-f(w)]\\
&\geq & f(w)-L_0\|(v-w)-\varepsilon \dir\| + \frac{f(w+\varepsilon \dir)-f(w)}{\varepsilon} \varepsilon\\
&\geq& f(w)+\frac{f(w+\varepsilon \dir)-f(w)}{\varepsilon} \varepsilon-L_0\|v-w\| \Big(\frac{v-w}{\|v-w\|}- \dir\Big).
\end{eqnarray*}
Taking $v\to w$ (which is equivalent to $\varepsilon \to 0$), we get
\begin{eqnarray*}
 f(v) 
 &\geq& f(w) + f^{\prime}(w;\dir)\varepsilon+o(\|v-w\|) \\
&\geq& f(w) + \langle g,\varepsilon \dir\rangle +o(\|v-w\|)\\
&=& f(w) + \langle g,v-w\rangle
+\varepsilon \Big\langle g,  \dir-\frac{(v-w)}{\varepsilon}\Big\rangle
+o(\|v-w\|)\\
&=& f(w) + \langle g,v-w\rangle
+o(\|v-w\|),
\end{eqnarray*}
where in the second step we used the starting assumption.
\end{proof}

Finally, we present the well-known fact that weak convexity implies that the variation of the function compared to its subgradient approximation is lower bounded by a negative quadratic.

\begin{prop}[Characterization of weak convexity from the regular subdifferential] \label{prop:wk_cvx_subdifferential}
Let $f:{\cal W}\mapsto \mathbb{R}$ be subdifferentiable and Lipschitz w.r.t.~$\|\cdot\|$. Then $f$ is $\rho$-weakly convex if and only if for all $w,v\in \mathbf{E}$, and $g\in \partial f(w)$
\begin{equation} \label{eqn:subgrad_weak_cvx}
 f(v)\geq f(w)+\langle g,v-w\rangle-\frac{\rho}{2}\|v-w\|^2.  
\end{equation}
\end{prop}

\begin{proof}
We prove both implications. For $\Rightarrow$, let $v,w\in \mathbf{E}$, and $0< \lambda<1$. By $\rho$-weak convexity:
\begin{align*}
f((1-\lambda)v+\lambda w) 
&\leq (1-\lambda) f(v)+\lambda f(w)+\frac{\rho\lambda(1-\lambda)}{2}\|v-w\|^2\\
\Longrightarrow \quad 
(1-\lambda)[f(v)-f(w)]
&\geq f((1-\lambda)v+\lambda w)-f(w)-\frac{\rho\lambda(1-\lambda)}{2}\|v-w\|^2\\
\Longrightarrow \quad 
f(v)-f(w)
&\geq \lim\inf_{\lambda\to 1} \Big[\frac{ f(w+(1-\lambda)(v-w))-f(w)}{(1-\lambda)}-\frac{\rho\lambda}{2}\|v-w\|^2\Big]\\
&= f^{\prime}(w;v-w)-\frac{\rho}{2}\|v-w\|^2\\
&\geq \langle g,v-w\rangle-\frac{\rho}{2}\|v-w\|^2,
\end{align*}
where in the last inequality we used Proposition \ref{prop:loc_subdiff}.

Next, for $\Leftarrow$, let $v,w\in\mathbf{E}$ and $0\leq \lambda\leq 1$. Then, letting $g\in \partial f((1-\lambda)w+\lambda v)$, and using \eqref{eqn:subgrad_weak_cvx} twice, we get
\begin{eqnarray*}
f(v) &\geq& f((1-\lambda)w+\lambda v)+\langle g,(1-\lambda)(v-w)\rangle -\frac{\rho}{2}\|(1-\lambda)(v-w)\|^2\\
f(w) &\geq& f((1-\lambda)w+\lambda v)+\langle g,\lambda(w-v)\rangle -\frac{\rho}{2}\|\lambda(v-w)\|^2.
\end{eqnarray*}
Multiplying the first inequality by $\lambda$ and the second one by $(1-\lambda)$, gives
\begin{eqnarray*}
\lambda f(v)+(1-\lambda)f(w)
&\geq& f((1-\lambda)w+\lambda v)-\frac{\rho\lambda(1-\lambda)}{2}\|v-w\|^2,
\end{eqnarray*}
which concludes the proof.
\end{proof}

From the previous proposition, we can easily conclude that any smooth function is weakly convex.
\begin{cor} \label{cor:smooth_implies_wk_cvx}
Let $f:{\cal W}\mapsto\mathbb{R}$ be a $L_1$-smooth function (i.e., $\|\nabla f(v)-\nabla f(w)\|_{\ast}\leq L_1\|v-w\|$, for all $v,w\in {\cal W}$). Then $f$ is $L_1$-weakly convex.
\end{cor}
\begin{proof}
Let $v,w\in {\cal W}$. Then by the Fundamental Theorem of Calculus:
\begin{align*}
    f(v) &= f(w) +\int_0^1 \langle \nabla f(w+s(v-w)),v-w\rangle ds \\
    &=  f(w) + \langle \nabla f(w),v-w \rangle+\int_0^1 \langle \nabla f(w+s(v-w))-\nabla f(w),v-w\rangle ds \\
    &\geq f(w) + \langle \nabla f(w),v-w \rangle-L_1\|v-w\|^2 \int_0^1 s ds .
\end{align*}
We conclude by Proposition \ref{prop:wk_cvx_subdifferential} that $f$ is $L_1$-weakly convex.
\end{proof}

\subsubsection{Basic Rules of the Subdifferential, Optimality Conditions and Stationarity Gap}

We know provide some basic tools regarding subdifferentials and optimality conditions in weakly convex programming, which will also allow us to introduce the notion of stationarity gap in this setting.

To start, we provide a basic calculus rule for the subdifferential of a sum of weakly convex functions.

\begin{thm}[Corollary 10.9 from \cite{RockaWets:1998}] \label{thm:sum_subdiff}
If $f:\mathbf{E}\mapsto\mathbb{R}$ be weakly convex, and $g:\mathbf{E}\mapsto\mathbb{R}
\cup\{+\infty\}$ be convex, lower semicontinuous, and such that $w\in\mbox{dom}(g)$. Then $\partial (f+g)(w)=\partial f(w)+\partial g(w).$
\end{thm}

Next, we provide a relation between directional derivatives and the regular subdifferential.

\begin{prop}[From Proposition 8.32 in \cite{RockaWets:1998}]
If $\varphi:\mathbb{E}\mapsto\mathbb{R}\cup\{+\infty\}$ is weakly convex, then \[\mbox{\em dist}(0,\partial \varphi(w))=-\inf_{\|\dir\|\leq 1} \varphi^{\prime}(w;\dir).\]
\end{prop}

With these results, we can now provide optimality conditions for weakly convex optimization

\begin{prop}[Stationarity conditions for weakly convex optimization] \label{prop:opt_cond_wk_cvx}
Let $f:{\cal W}\mapsto\mathbb{R}$ be $\rho$-weakly convex and $L_0$-Lipschitz w.r.t.~$\|\cdot\|$, and ${\cal W}$ a closed and convex set. Then, if $w^{\ast}\in\arg\min\{ f(w):\, w\in{\cal W}\},$ then there exists $g\in \partial f(w^{\ast})$ such that
\[ \langle g,v-w^{\ast}\rangle \geq 0 \qquad (\forall v\in{\cal W}). \]
\end{prop}

\begin{proof}
First, we observe that without loss of generality, $f:\mathbf{E}\mapsto\mathbb{R}$ (this is a consequence of the Lipschitz extension Theorem). Let now $g(w)=\chi_{\cal W}(w)$ (i.e., the convex indicator function, as defined in the beginning of this section). Since $w^{\ast}\in {\cal W}$, by Proposition \ref{thm:sum_subdiff}, we have
$\partial (f+g)(w^{\ast})=\partial f(w^{\ast})+\partial g(w^{\ast}).$ Now we apply Proposition \ref{prop:opt_cond_wk_cvx} to $\varphi(w)=f(w)+g(w)$; since $w^{\ast}$ is a minimizer of $\varphi$, we have that $\varphi^{\prime}(w^{\ast};\dir)\geq 0$ for all $\dir$, and hence $\mbox{dist}(0,\partial \varphi(w^{\ast}))=0$. Since $\partial g(w^{\ast})={\cal N}(w^{\ast})$, we get that 
\[ 0= \mbox{dist}(0,\partial f(w^{\ast})+{\cal N}_{\cal W}(w^{\ast})), \]
and this implies that there exists $g\in \partial f(w^{\ast})$, such that $g\in -{\cal N}_{\cal W}(w^{\ast})$, i.e.,
\[ \langle g,v-w^{\ast}\rangle \geq 0 \qquad(\forall v\in {\cal W}). \]
\end{proof}

The previous result leads to a natural definition of the stationarity gap in weakly convex optimization:
\begin{equation}\label{eqn:stat_gap_wk_cvx} \gap_f(w) = \inf_{g\in\partial f(w)} \sup_{v\in {\cal W}} \langle g,v-w\rangle. 
\end{equation}
Notice that, by Proposition \ref{prop:opt_cond_wk_cvx}, any minimizer of a weakly convex and Lipschitz function is such that its stationarity gap is equal to zero.

\subsection{Missing proofs from Section \ref{sec:prox_operator}} \label{sec:app-prox}

\subsubsection{Missing Details in Consequences of Proximal Near Stationarity}

\label{sec:PNS_conseq}

Now we explain some technical details behind the derivation of the following consequence for proximal nearly-stationary algorithms
\begin{equation} \label{eqn:PNS_conseq}
\mathbb{E}_{S\sim{\cal D}^n,{\cal A}}\big[\|\prox_{F_{\cal D}}^{\beta}({\cal A}(S))-{\cal A}(S)\| \big] \leq \vartheta \qquad\mbox{ and }\qquad 
\mathbb{E}_{S\sim{\cal D}^n,{\cal A}} \big[\gap_{F_{\cal D}}\big(\prox_{F_{\cal D}}^{\beta}({\cal A}(S))\big)\big] \leq \vartheta.
\end{equation}
First, we suppose ${\cal A}$ is $(\vartheta,\beta)$-proximal nearly stationary. From this, we directly conclude the first property,
\[ \mathbb{E}_{S\sim{\cal D}^n,{\cal A}}\big[\|\prox_{F_{\cal D}}^{\beta}({\cal A}(S))-{\cal A}(S)\| \big] \leq \vartheta.  \]
For the second property, we first recall the stationarity gap in weakly convex
optimization (see eqn.~\eqref{eqn:stat_gap_wk_cvx}): here, for  $w\in {\cal W}$ and objective $f:{\cal W}\mapsto\mathbb{R}$, define
\[ \gap_{f}(w) = \inf_{g\in \partial f(w)}\sup_{v\in {\cal W}} \langle g,w-v\rangle. \]
Now, if  ${\cal B}:\cZ^n\mapsto \mathbb{R}$ is a randomized algorithm, its expected gap corresponds to
\[ \mathbb{E}_{S\sim {\cal D}^n, {\cal B}} [\gap_{F_{\cal D}}({\cal B}(S))] = \mathbb{E}_{S\sim {\cal D}^n, {\cal B}} \Big[ \inf_{g\in \partial F_{\cal D}({\cal B}(S))}\sup_{v\in {\cal W}} \langle g,{\cal B}(S)-v\rangle\Big]. \]

Finally, under this definition of the expected gap, we have that if ${\cal B}(S)=\prox_{F_{\cal D}}^{\beta}({\cal A}(S))$, then by Lemma~\ref{lem:prox_estimate} and $(\vartheta,\beta)$-proximal near stationarity,
\begin{align*}
    \gap_{F_{\cal D}}({\cal B}) &= \mathbb{E}_{S\sim {\cal D}^n, {\cal B}} \Big[ \inf_{g\in \partial F_{\cal D}({\cal B}(S))}\sup_{v\in {\cal W}} \langle g,{\cal B}(S)-v\rangle\Big]
    \leq \mathbb{E}_{S\sim {\cal D}^n, {\cal B}} \Big[ \beta D\|{\cal B}(S)-{\cal A}(S)\|\Big]\\
    &\leq \vartheta,
\end{align*}
concluding the claim.

\end{document}